\documentclass[10pt]{article}

\usepackage[T1]{fontenc}
\usepackage[utf8]{inputenc}
\usepackage{graphicx}
\usepackage{float}
\usepackage{amsmath}
\usepackage{amssymb}
\usepackage{mathtools}
\usepackage{amsthm}
\usepackage{booktabs}
\usepackage[margin=1.1in]{geometry}
\usepackage[hidelinks]{hyperref}
\usepackage{natbib}
\usepackage{subcaption}
\usepackage{enumitem}
\theoremstyle{plain}
\newtheorem{theorem}{Theorem}
\newtheorem{proposition}{Proposition}
\newtheorem{lemma}{Lemma}
\newtheorem{corollary}{Corollary}
\theoremstyle{definition}

\theoremstyle{definition}

\title{{\Large\bfseries
Fisher Rank Inflation: A Spectral Signature of Memorization under Label Noise
}}

\author{
Satwik Bathula, Anand A. Joshi \\[0.5em]
Department of Electrical and Computer Engineering, University of Southern California \\[0.25em]
\texttt{sbathula@usc.edu, ajoshi@usc.edu}
}

\date{}
\begin{document}

\maketitle

\begin{abstract}
Deep networks trained with label noise often learn clean structure before memorizing corrupted labels. We show that this transition leaves a spectral signature in the centered scatter of per-example last-layer gradients. Its effective rank transiently expands during memorization and contracts after corrupted labels are fit. We call this phenomenon \emph{Fisher Rank Inflation}.

 We show that corrupted labels can increase effective rank by injecting spectral mass into low-energy or previously unused eigendirections, thereby increasing the entropy of the gradient spectrum. We derive a first-order leave-one-out attribution formula and identify conditions under which corrupted examples contribute more strongly to rank inflation than clean examples. We further show that once the normalized Fisher-gradient spectrum stabilizes, individual attribution signals vanish, explaining the post-memorization weakening of leave-one-out rank contributions.

We empirically test these mechanistic predictions on CIFAR-10 and CIFAR-100 using SmallCNN, ResNet18, and Vision Transformers under symmetric label corruption, and additionally evaluate the phenomenon on CIFAR-10N with naturally occurring human annotation errors. Across datasets and architectures, Fisher effective rank exhibits a consistent inflation--collapse trajectory aligned with memorization dynamics. At peak-rank checkpoints, corrupted examples are strongly enriched among the highest rank-contributing samples, with top-100 noisy fractions ranging from \(69.2\%\) to \(96.2\%\) across five-seed experiments under synthetic corruption and reaching \(94.4\%\pm1.9\%\) on CIFAR-10N. The first-order spectral attribution closely matches exact leave-one-out rank contributions in the convolutional models and remains enriched in the Vision Transformer. In addition, a seeded corruption sweep shows that peak Fisher effective rank increases monotonically with corruption severity, rising from \(28.88 \pm 1.95\) under clean training to \(97.09 \pm 1.78\) at \(60\%\) corruption. In several settings, the retrospectively identified onset of rank inflation precedes observable test degradation. The persistence of Fisher Rank Inflation under both synthetic corruption and naturally occurring human annotation errors suggests that the phenomenon captures a broader spectral signature of memorization rather than an artifact of a particular noise-generation process.

These results establish Fisher Rank Inflation as a spectral signature of memorization under label noise and connect the dynamics of the last-layer Fisher-gradient spectrum to corrupted-example enrichment, corruption severity, and the transition from structure learning to memorization.
\end{abstract}

\section{Introduction}

Deep neural networks can achieve near-perfect accuracy even when a substantial fraction of training samples are corrupted. A large body of work has shown that learning under label noise typically proceeds in stages. Networks first learn structured patterns shared across many examples before eventually memorizing corrupted labels that are inconsistent with the underlying data distribution. This phenomenon has motivated extensive research on memorization, early learning, label-noise robustness, and data attribution.

 Existing analyses often focus on training and test losses, prediction dynamics, representation geometry, or example-level influence. While these perspectives characterize when memorization occurs, they provide less insight into how corrupted examples reshape the spectrum of training gradients during the memorization process. The gradient spectrum captures how optimization directions are distributed across the training set, providing a global geometric view of memorization that complements example-level or loss-based analyses.

In this work, we study memorization through the covariance structure of per-example last-layer gradients. Specifically, we analyze the effective rank of the centered Fisher-gradient scatter matrix. Empirically, we observe a consistent phenomenon across datasets and architectures. As training transitions from structured learning to memorization, the effective rank of the Fisher-gradient scatter undergoes a pronounced expansion. This inflation peaks during memorization and subsequently collapses once corrupted labels have been fit. We refer to this phenomenon as \emph{Fisher Rank Inflation}.

Beyond this effect, we use leave-one-out rank attribution to identify which training examples contribute to inflation. For each example \(i\), we compute
\[
\Delta_i=\operatorname{er}(S)-\operatorname{er}(S_{-i}),
\]
where \(S_{-i}\) is the centered Fisher-gradient scatter after removing example \(i\) and re-centering the remaining gradients. Examples with large positive \(\Delta_i\) are those whose removal most reduces effective rank, and therefore contribute most to rank inflation. At the peak of Fisher Rank Inflation, these high-\(\Delta_i\) examples are disproportionately corrupted. Removing a corrupted example typically produces a larger reduction in effective rank than removing a clean example.

To explain these observations, we develop a spectral theory of Fisher Rank Inflation. We show that corrupted labels can increase effective rank by introducing covariance mass into low-energy or previously unused eigendirections, thereby increasing spectral entropy. We further derive a first-order leave-one-out attribution formula that identifies the spectral conditions under which corrupted examples contribute more strongly to rank inflation than clean examples. Finally, we establish conditions under which these attribution signals weaken after memorization as the normalized Fisher-gradient spectrum becomes stable.

We validate the theory under symmetric label corruption on CIFAR-10 and CIFAR-100 using SmallCNN, ResNet18, and Vision Transformers. Across datasets and architectures, we observe a consistent inflation--collapse trajectory. At seed-specific peak-rank checkpoints, corrupted examples are strongly enriched among the highest rank-contributing samples, with top-100 noisy fractions ranging from \(69.2\%\) to \(96.2\%\) across five-seed experiments.  A seeded corruption sweep further shows that peak Fisher effective rank increases monotonically with corruption severity, rising from \(28.88 \pm 1.95\) under clean training to \(97.09 \pm 1.78\) at \(60\%\) corruption. In several settings, the retrospectively identified onset of Fisher Rank Inflation occurs before observable test degradation.

Our contributions are summarized as follows:

\begin{itemize}
\item We identify \emph{Fisher Rank Inflation}, a spectral phenomenon in which the effective rank of the Fisher-gradient scatter expands during memorization and collapses afterward.

\item We explain how corrupted labels inflate effective rank through spectral spreading and derive conditions under which corrupted examples exhibit larger leave-one-out rank contributions.

\item We explain the post-memorization weakening of attribution signals as a consequence of normalized spectral stabilization.

\item We validate the phenomenon across multiple architectures, datasets, and
noise regimes, including symmetric synthetic label corruption on CIFAR-10 and
CIFAR-100 and naturally occurring human annotation noise on CIFAR-10N. 
Across these settings, we observe consistent inflation--collapse dynamics, strong corrupted-example enrichment at peak-rank checkpoints, and monotonic growth of peak Fisher effective rank with corruption severity.
\end{itemize}

\section{Related Work}

\paragraph{Memorization and generalization in deep networks.}
Deep networks are known to fit even randomly labeled data, challenging classical accounts of generalization based purely on capacity control~\citep{zhang2017understandingdeeplearningrequires}. Subsequent work showed that memorization is not uniform throughout training. Networks tend to learn simple, structured patterns before fitting noisy or exceptional examples~\citep{pmlr-v70-arpit17a,rolnick2018deeplearningrobustmassive,toneva2019empiricalstudyexampleforgetting}. Memorization has also been connected to long-tailed structure in natural data distributions~\citep{10.1145/3357713.3384290,NEURIPS2020_1e14bfe2}. Related phenomena such as double descent further show that interpolation and generalization can interact in non-classical ways~\citep{Belkin2019Reconciling,Nakkiran_2021}. Our work studies memorization from a different perspective, focusing not on loss, accuracy, or example forgetting, but on how memorization changes the spectrum of the centered covariance of per-example last-layer gradients.

\paragraph{Learning with noisy labels.}
A large body of work studies robust training under label noise. Many approaches exploit the empirical observation that clean examples are learned earlier than corrupted ones. Co-teaching and Co-teaching+ select small-loss examples across peer networks~\citep{Han2018CoTeaching,pmlr-v97-yu19b}, while related methods use agreement, co-regularization, semi-supervised learning, or early-learning regularization to prevent noisy-label memorization~\citep{Wei_2020_CVPR,li2020dividemixlearningnoisylabels,Liu2020EarlyLearningRegularization}. Other approaches modify losses or estimate label-noise transition structure, including loss correction, generalized cross entropy, symmetric cross entropy, bootstrapping, and mixture-based label-noise modeling~\citep{Patrini_2017_CVPR,Zhang2018GeneralizedCrossEntropy,Wang_2019_ICCV,reed2015trainingdeepneuralnetworks,pmlr-v97-arazo19a}. Dataset-level label-quality methods such as Confident Learning, AUM, and CleanNet aim to identify or mitigate mislabeled samples~\citep{Northcutt2021ConfidentLearning,Pleiss2020IdentifyingMislabeledData,Lee_2018_CVPR}. In contrast to these robust-training methods, our goal is not to propose a new noise-robust optimizer or loss. Instead, we identify a mechanism by which corrupted labels reshape Fisher-gradient scatter during memorization.

\paragraph{Fisher information, curvature, and spectral structure.}
Fisher information and natural-gradient geometry have long been used to study learning dynamics and optimization~\citep{6790500}. Practical approximations such as K-FAC and Kronecker-factored Fisher methods exploit structure in the Fisher matrix for scalable second-order optimization~\citep{pmlr-v37-martens15,pmlr-v48-grosse16,George2018FastApproximateNaturalGradient}. Several works have analyzed spectral properties of Fisher or curvature matrices in neural networks, including Fisher eigenvalue statistics, random-matrix analyses of Fisher spectra, and Hessian spectra in overparameterized models~\citep{pmlr-v89-karakida19a,Pennington2018SpectrumFisher,sagun2018empiricalanalysishessianoverparametrized,papyan2019spectrumdeepnethessiansscale}. Related spectral perspectives study heavy-tailed self-regularization, loss-landscape geometry, and terminal representation collapse~\citep{pmlr-v97-mahoney19a,fort2019emergentpropertieslocalgeometry,Papyan2020NeuralCollapse}. Our work differs by studying the training-time dynamics of the centered covariance of per-example last-layer gradients under label corruption. We show that memorization induces a transient effective-rank expansion of this gradient spectrum, followed by collapse after corrupted labels are fit.

\paragraph{Data attribution and influence.}
Example-level attribution methods seek to identify which training examples are responsible for model behavior. Influence functions trace predictions back to training data through local sensitivity of the learned model~\citep{pmlr-v70-koh17a}. Subsequent work developed scalable or alternative attribution methods such as TracIn, representer points, Data Shapley, datamodels, TRAK, and FastIF~\citep{Pruthi2020EstimatingTrainingDataInfluence,Yeh2018RepresenterPointSelection,pmlr-v97-ghorbani19c,ilyas2022datamodelspredictingpredictionstraining,park2023trakattributingmodelbehavior,guo-etal-2021-fastif}. Other work has highlighted fragility and limitations of influence estimates in deep networks~\citep{basu2021influencefunctionsdeeplearning}. Our leave-one-out analysis asks a different question. Rather than attributing a prediction or loss to a training example, it attributes changes in the effective rank of the Fisher-gradient scatter, yielding a first-order spectral attribution of each example's contribution to Fisher Rank Inflation.

\paragraph{Spectral bias and training dynamics.}
Our work is also related to analyses of spectral bias and overparameterized training dynamics. Neural networks often learn low-frequency or simple functions before more complex patterns~\citep{pmlr-v97-rahaman19a}, and theoretical analyses based on the neural tangent kernel and overparameterization have clarified aspects of optimization and interpolation~\citep{Jacot2018NeuralTangentKernel,pmlr-v97-allen-zhu19a}. These works characterize biases in function-space or parameter-space training dynamics. Rank Inflation instead reflects a transient spectral reorganization of the gradient covariance. Corrupted labels inject covariance mass into weak or previously unused eigendirections, increasing Fisher effective rank during memorization and producing a peak-time attribution signal for corrupted examples.

\section{A Spectral Mechanism for Fisher-Rank Inflation}
\label{sec:theory}

We explain Fisher-rank inflation under label noise via a spectral mechanism. Specifically, we formalize the following empirical observations:
\begin{enumerate}[label=(\roman*)]
    \item Label corruption increases the effective rank of the last-layer gradient covariance.
    \item Corrupted samples have larger leave-one-out rank contributions.
    \item This attribution signal weakens after memorization.
\end{enumerate}

The results in this section should be interpreted as local spectral statements
at a fixed training checkpoint. They identify
spectral conditions under which corrupted-label gradients increase effective
rank, receive larger leave-one-out rank attribution, and eventually lose
individual attribution strength once the normalized gradient spectrum becomes
stable. The training-time inflation--collapse trajectory studied in Section~4
is therefore an empirical phenomenon explained by, and consistent with, these
local spectral mechanisms. All the proofs have been deferred to the appendix.

\subsection{Setup}

Let \(f_\theta(x)\in\mathbb{R}^K\) denote the logits of a \(K\)-class neural network and let
\(\hat p_i=\mathrm{softmax}(f_\theta(x_i))\) be the predicted class-probability vector for training example
\((x_i,y_i)\). We focus on the final affine classification layer. Let \(z_i\in\mathbb{R}^m\) denote the
penultimate representation of \(x_i\). To include both final-layer weights and biases in a single notation,
define the augmented representation
\[
\tilde z_i =
\begin{bmatrix}
z_i \\
1
\end{bmatrix}
\in\mathbb{R}^{m+1}.
\]
Under softmax cross-entropy, the per-example gradient with respect to the final affine layer, including
both weights and biases, can be written as
\[
g_i = (\hat p_i - e_{y_i})\otimes \tilde z_i
\in \mathbb{R}^{K(m+1)}.
\]
Thus \(g_i\) is the vectorized final-layer gradient used throughout both the theory and experiments. In
particular, the bias-gradient component is represented by the final coordinate of \(\tilde z_i\). We set
\[
d = K(m+1),
\]
and let \(G\in\mathbb{R}^{n\times d}\) denote the matrix whose rows are the per-example gradients
\(g_i^\top\).

Following the experiments, we work with centered gradients
\[
\bar g_i = g_i - \mu,
\qquad
\mu = \frac{1}{n}\sum_{i=1}^n g_i,
\]
and define the centered empirical scatter matrix
\[
S = \sum_{i=1}^n \bar g_i \bar g_i^\top.
\]

We refer to \(S\) as the centered Fisher-gradient scatter matrix. This
object is closely related to the empirical Fisher because it is constructed
from per-example loss gradients. However, it is not the conventional
uncentered empirical Fisher matrix, which would take the form
\(\sum_i g_i g_i^\top\), or its normalized analogue. The centering step is
deliberate. Our goal is to study the spectral dispersion of gradient
fluctuations across examples, rather than the total second moment of the
gradients. Since entropy effective rank is invariant to positive scalar
rescaling, the distinction between using the scatter \(S\), the covariance
\(S/n\), or the equivalent centered Gram matrix does not affect the
effective-rank values studied below.

We use the term ``scatter matrix'' throughout. Equivalently, \(S/n\) is the empirical covariance. Since
effective rank is invariant to positive rescaling, using \(S\), \(S/n\), or \(\bar G^\top \bar G\) gives the same
effective-rank value.

Let \(C\) and \(N\) denote the clean-label and corrupted-label subsets. For the additive
decomposition of the full centered scatter, we use globally centered clean and corrupted scatters:
\[
S_C^{\mathrm{glob}}
=
\sum_{i\in C}(g_i-\mu)(g_i-\mu)^\top,
\qquad
S_N^{\mathrm{glob}}
=
\sum_{i\in N}(g_i-\mu)(g_i-\mu)^\top.
\]
Then
\[
S = S_C^{\mathrm{glob}} + S_N^{\mathrm{glob}}.
\]

For the noisy-to-clean effective-rank diagnostic used in the experiments, however, we center each subset
separately. Define
\[
\mu_C = \frac{1}{|C|}\sum_{i\in C}g_i,
\qquad
\mu_N = \frac{1}{|N|}\sum_{i\in N}g_i,
\]
and
\[
\widetilde S_C
=
\sum_{i\in C}(g_i-\mu_C)(g_i-\mu_C)^\top,
\qquad
\widetilde S_N
=
\sum_{i\in N}(g_i-\mu_N)(g_i-\mu_N)^\top.
\]
The subset-centered scatters \(\widetilde S_C\) and \(\widetilde S_N\) measure within-group spectral
dispersion and are distinct from the globally centered additive decomposition
\(S=S_C^{\mathrm{glob}}+S_N^{\mathrm{glob}}\).

Let \(\lambda_1(S),\ldots,\lambda_r(S)\) be the positive eigenvalues of \(S\), where \(r=\mathrm{rank}(S)\).
The normalized positive spectrum is
\[
p_k(S)=\frac{\lambda_k(S)}{\mathrm{tr}(S)}.
\]
The entropy effective rank is
\[
\mathrm{er}(S)
=
\exp\left(
-\sum_{k=1}^r p_k(S)\log p_k(S)
\right).
\]
Throughout, we assume that every scatter matrix whose effective rank is evaluated has positive trace.

\subsection{Sensitivity of Effective Rank}

\begin{lemma}[Effective-rank derivative]
\label{lem:effective-rank-derivative}
Let \(S\) have positive simple eigenvalues
\(\lambda_1,\ldots,\lambda_r\), trace \(T=\operatorname{tr}(S)\), normalized spectrum
\(p_k=\lambda_k/T\), entropy
\[
H(S) = -\sum_{k=1}^r p_k\log p_k,
\]
and effective rank \(\operatorname{er}(S)=e^{H(S)}\). Then
\[
\frac{\partial \operatorname{er}(S)}{\partial \lambda_k}
=
-\frac{\operatorname{er}(S)}{T}
\bigl(\log p_k + H(S)\bigr).
\]
Consequently,
\[
\frac{\partial \operatorname{er}(S)}{\partial \lambda_k}>0
\quad \Longleftrightarrow \quad
p_k < \frac{1}{\operatorname{er}(S)}.
\]
\end{lemma}

Define the spectral sensitivity coefficient
\[
\beta_k(S)
=
-\frac{\operatorname{er}(S)}{\operatorname{tr}(S)}
\bigl(\log p_k(S)+H(S)\bigr).
\]
Eigenvalue-coordinate directions \(u_k u_k^\top\) with \(\beta_k(S)>0\)
are infinitesimal effective-rank-increasing directions.

\subsection{Centered Leave-One-Out Attribution}
\label{subsec:centered-loo}

For a training example \(i\), define the leave-one-out scatter matrix by removing
example \(i\) and re-centering the remaining gradients. Specifically, let
\[
\mu_{-i}
=
\frac{1}{n-1}\sum_{j\neq i} g_j,
\]
and define
\[
S_{-i}
=
\sum_{j\neq i}
(g_j-\mu_{-i})(g_j-\mu_{-i})^\top.
\]
The leave-one-out Fisher-rank contribution of example \(i\) is
\[
\Delta_i
=
\operatorname{er}(S)-\operatorname{er}(S_{-i}).
\]
Thus \(\Delta_i>0\) means that removing example \(i\) decreases effective rank,
so example \(i\) contributes positively to Fisher-rank inflation.

\begin{lemma}[Centered deletion identity]
\label{lem:centered-deletion}
Let
\[
S
=
\sum_{j=1}^n \bar g_j\bar g_j^\top,
\qquad
\bar g_j
=
g_j-\mu,
\qquad
\mu
=
\frac{1}{n}\sum_{j=1}^n g_j.
\]
Let \(S_{-i}\) be the centered scatter matrix obtained by removing example \(i\)
and re-centering the remaining \(n-1\) gradients. Then
\[
S_{-i}
=
S-\frac{n}{n-1}\bar g_i\bar g_i^\top.
\]
\end{lemma}

\begin{lemma}[First-order LOO expansion]
\label{lem:first-order-loo}
Assume the positive eigenvalues of \(S\) are simple. Let
\[
S
=
\sum_{k=1}^r \lambda_k u_k u_k^\top,
\qquad
\lambda_1(S)>\lambda_2(S)>\cdots>\lambda_r(S)>0.
\]
Define the local spectral gap
\[
\delta(S)
=
\min\left\{
\lambda_r(S),
\min_{1\le j<k\le r}|\lambda_j(S)-\lambda_k(S)|
\right\}.
\]
When \(r=1\), the second minimum is interpreted as \(+\infty\), so
\[
\delta(S)=\lambda_1(S).
\]

Since \(S=\sum_j \bar g_j\bar g_j^\top\), each centered gradient \(\bar g_i\)
lies in \(\operatorname{range}(S)\). Let
\[
a=\frac{n}{n-1}.
\]
Assume the deletion perturbation is sufficiently small relative to the local
spectral gap:
\[
a\|\bar g_i\|^2 \le c\delta(S)
\]
for a sufficiently small universal constant \(c>0\). Then
\[
\Delta_i
=
a\sum_{k=1}^r
\beta_k(S)(u_k^\top \bar g_i)^2
+
R_i,
\]
where
\[
|R_i|\le C(S)\|\bar g_i\|^4,
\]
for a constant \(C(S)\) depending on a local spectral neighborhood controlled by
\(\delta(S)\).
\end{lemma}

\subsection{Entropy Gain from New Spectral Directions}

\begin{lemma}[Entropy gain from new spectral directions]
\label{lem:new-directions-entropy}
Let \(A\succeq 0\) have positive trace and rank \(r\), and let \(B\succeq 0\).
Let \(P_A\) be the orthogonal projector onto \(\operatorname{range}(A)\), and
\(Q_A=I-P_A\). If
\[
\operatorname{tr}(Q_A B Q_A)>0,
\]
then there exists \(\alpha_0>0\) such that for all \(\alpha\in(0,\alpha_0)\),
\[
\operatorname{er}(A+\alpha B)>\operatorname{er}(A).
\]
\end{lemma}

\begin{lemma}[Projected corrupted scatter under random corruption]
\label{lem:random-corruption-projected-scatter}
Condition on a fixed training checkpoint and fixed prediction-representation pairs
\[
\{(\hat p_i,\tilde z_i)\}_{i=1}^n.
\]
Let \(P_C\) be the orthogonal projector onto
\(\operatorname{range}(S_C^{\mathrm{glob}})\), and let \(Q_C=I-P_C\).
Suppose the corrupted labels are sampled randomly according to a corruption
distribution, producing corrupted gradients
\[
g_i(\tilde y_i)
=
(\hat p_i-e_{\tilde y_i})\otimes \tilde z_i,
\qquad i\in N.
\]
Let \(\mu\) denote the global gradient mean induced by the same corrupted-label
realization, and define
\[
S_N^{\mathrm{glob}}
=
\sum_{i\in N}
(g_i(\tilde y_i)-\mu)(g_i(\tilde y_i)-\mu)^\top.
\]

The expectation is only over the random replacement labels
\(\{\tilde y_i\}_{i\in N}\), conditional on the fixed checkpoint and fixed
prediction-representation pairs \(\{(\hat p_i,\tilde z_i)\}_{i=1}^n\). Then
\[
\mathbb{E}\left[
\operatorname{tr}(Q_C S_N^{\mathrm{glob}} Q_C)
\right]
=
\sum_{i\in N}
\mathbb{E}\left[
\|Q_C(g_i(\tilde y_i)-\mu)\|^2
\right].
\]
Consequently,
\[
\mathbb{E}\left[
\operatorname{tr}(Q_C S_N^{\mathrm{glob}} Q_C)
\right]>0
\]
whenever at least one corrupted centered gradient has a nonzero component
outside the globally centered clean-gradient span with positive probability.
\end{lemma}

\subsection{Rank Inflation Under Label Corruption}

To analyze how corrupted gradients perturb the full Fisher-gradient scatter, we
use the globally centered decomposition
\[
S=S_C^{\mathrm{glob}}+S_N^{\mathrm{glob}}.
\]
Consider the interpolation
\[
S_\alpha
=
S_C^{\mathrm{glob}}+\alpha S_N^{\mathrm{glob}},
\qquad
\alpha\ge 0.
\]

Lemma~\ref{lem:random-corruption-projected-scatter} provides an expected-case
justification for the Case 1 condition in Theorem~\ref{thm:local-rank-inflation}.
It shows that
\[
\mathbb{E}\left[
\operatorname{tr}(Q_C S_N^{\mathrm{glob}} Q_C)
\right]>0
\]
whenever random label corruption produces corrupted centered gradients with
nonzero projection outside the globally centered clean-gradient span. In our
experimental setting, corrupted labels are drawn uniformly from
\([K]\setminus\{y_i\}\). Since the augmented representations \(\tilde z_i\) are high-dimensional
and the corrupted label directions vary across \(K-1\) incorrect classes, this
nondegeneracy condition is expected to hold generically. Theorem~\ref{thm:local-rank-inflation}
is stated deterministically for the realized scatter matrices.

\begin{theorem}[Local Fisher-rank inflation]
\label{thm:local-rank-inflation}
Let \(P_C\) be the orthogonal projector onto
\(\operatorname{range}(S_C^{\mathrm{glob}})\), and let \(Q_C=I-P_C\).

There are two cases.

\textbf{Case 1: corrupted gradients create new spectral directions.}
If
\[
\operatorname{tr}(Q_C S_N^{\mathrm{glob}} Q_C)>0,
\]
then there exists \(\alpha_0>0\) such that for all \(\alpha\in(0,\alpha_0)\),
\[
\operatorname{er}(S_C^{\mathrm{glob}}+\alpha S_N^{\mathrm{glob}})
>
\operatorname{er}(S_C^{\mathrm{glob}}).
\]

\textbf{Case 2: corrupted gradients lie in the clean-gradient span.}
Assume additionally that \(S_C^{\mathrm{glob}}\) has simple positive eigenvalues
\[
S_C^{\mathrm{glob}}
=
\sum_{k=1}^r \lambda_k u_k u_k^\top .
\]
If
\[
\operatorname{range}(S_N^{\mathrm{glob}})
\subseteq
\operatorname{range}(S_C^{\mathrm{glob}})
\]
and
\[
\sum_{k=1}^r
\beta_k(S_C^{\mathrm{glob}})
u_k^\top S_N^{\mathrm{glob}} u_k
>0,
\]
then there exists \(\alpha_0>0\) such that for all \(\alpha\in(0,\alpha_0)\),
\[
\operatorname{er}(S_C^{\mathrm{glob}}+\alpha S_N^{\mathrm{glob}})
>
\operatorname{er}(S_C^{\mathrm{glob}}).
\]
\end{theorem}

Theorem~\ref{thm:local-rank-inflation} shows that label corruption can inflate
Fisher effective rank in two ways. It can either create new spectral directions
outside the globally centered clean-gradient span, or it can place sufficient
covariance mass inside existing low-energy directions to which effective rank is
positively sensitive.

\subsection{Spectral Condition for Larger Corrupted-Sample LOO Contributions}
\label{subsec:loo-spectral-condition}

Lemma~\ref{lem:first-order-loo} shows that, to first order, the LOO contribution
of sample \(i\) is controlled by a quadratic form in its centered gradient. Define
the effective-rank sensitivity matrix
\[
B_S
=
\sum_{k=1}^r \beta_k(S)u_k u_k^\top.
\]
Then the first-order attribution score of sample \(i\) is
\[
\mathcal A_i(S)
=
\bar g_i^\top B_S \bar g_i
=
\sum_{k=1}^r \beta_k(S)(u_k^\top \bar g_i)^2.
\]
The matrix \(B_S\) has positive eigenvalues along infinitesimal
effective-rank-increasing eigendirections, negative eigenvalues along
effective-rank-decreasing eigendirections, and zero eigenvalues on
\(\ker(S)\).

Unless otherwise stated, \(\mathbb{E}[\cdot \mid i \in A]\) denotes the empirical
expectation over a uniformly chosen index from group \(A\).

For \(A\in\{C,N\}\), define the group second-moment matrix
\[
\Sigma_A
=
\mathbb{E}\left[\bar g_i\bar g_i^\top \mid i\in A\right].
\]

\begin{theorem}[Spectral condition for corrupted-sample LOO enrichment]
\label{thm:corrupted-larger-loo}
Assume the first-order LOO expansion of Lemma~\ref{lem:first-order-loo}
applies uniformly to all samples under consideration. One sufficient condition is that the deletion perturbations satisfy
\[
\frac{n}{n-1}\|\bar g_i\|^2 \le c\delta(S)
\]
for all \(i\) in the clean and corrupted groups. Let
\[
a=\frac{n}{n-1},
\qquad
\eta=\max_i\|\bar g_i\|^2.
\]
Here \(\|\cdot\|\) denotes the Euclidean norm, so the remainder bound satisfies
\[
|R_i|\le C(S)\|\bar g_i\|^4\le C(S)\eta^2.
\]
If
\[
\operatorname{tr}\left(B_S(\Sigma_N-\Sigma_C)\right)
>
\frac{2C(S)}{a}\eta^2,
\]
then
\[
\mathbb{E}[\Delta_i\mid i\in N]
>
\mathbb{E}[\Delta_i\mid i\in C].
\]
\end{theorem}

Theorem~\ref{thm:corrupted-larger-loo} does not claim that corrupted samples
always have larger LOO contribution. Rather, it identifies the spectral condition
under which the empirical enrichment occurs: corrupted gradients must have
larger covariance, relative to clean gradients, in directions weighted positively
by the effective-rank sensitivity matrix \(B_S\).

\begin{corollary}[Rank-increasing subspace sufficient condition]
\label{cor:rank-increasing-subspace}
Let
\[
K_+(S)=\{k:\beta_k(S)>0\}.
\]
Assume
\[
K_+(S)\neq\varnothing.
\]
Define
\[
U_+=\sum_{k\in K_+(S)}u_k u_k^\top,
\]
and let
\[
b_+=\min_{k\in K_+(S)}\beta_k(S)>0.
\]
Let
\[
a=\frac{n}{n-1},
\qquad
\eta=\max_i\|\bar g_i\|^2.
\]
Assume the perturbative condition
\[
\frac{n}{n-1}\|\bar g_i\|^2 \le c\delta(S)
\]
holds for all samples under consideration.

Suppose there exists \(\delta_+>0\) such that
\[
\mathbb{E}\left[\|U_+\bar g_i\|^2\mid i\in N\right]
-
\mathbb{E}\left[\|U_+\bar g_i\|^2\mid i\in C\right]
\ge
\delta_+.
\]
Suppose also that the contribution from rank-decreasing directions is bounded
below by \(-B_-\):
\[
\sum_{k\notin K_+(S)}
\beta_k(S)
\left(
\mathbb{E}\left[(u_k^\top \bar g_i)^2\mid i\in N\right]
-
\mathbb{E}\left[(u_k^\top \bar g_i)^2\mid i\in C\right]
\right)
\ge
-B_-.
\]
If
\[
b_+\delta_+ - B_- > \frac{2C(S)}{a}\eta^2,
\]
then
\[
\mathbb{E}[\Delta_i\mid i\in N]
>
\mathbb{E}[\Delta_i\mid i\in C].
\]
\end{corollary}

Corollary~\ref{cor:rank-increasing-subspace} gives an interpretable sufficient
condition for the spectral criterion in Theorem~\ref{thm:corrupted-larger-loo}.
It says that corrupted examples have larger expected LOO contribution when they
carry more energy in rank-increasing eigendirections, and when the contribution
from rank-decreasing directions does not dominate.

\subsection{Noisy-to-Clean Effective-Rank Ratio}

The noisy-to-clean effective-rank ratio used in the experiments is computed from
subset-centered clean and corrupted gradient scatter matrices:
\[
NIR
=
\frac{\operatorname{er}(\widetilde S_N)}
{\operatorname{er}(\widetilde S_C)}.
\]
Here
\[
\mu_C
=
\frac{1}{|C|}\sum_{i\in C} g_i,
\qquad
\mu_N
=
\frac{1}{|N|}\sum_{i\in N} g_i,
\]
and
\[
\widetilde S_C
=
\sum_{i\in C}(g_i-\mu_C)(g_i-\mu_C)^\top,
\qquad
\widetilde S_N
=
\sum_{i\in N}(g_i-\mu_N)(g_i-\mu_N)^\top.
\]
Thus \(NIR\) measures the within-group spectral dispersion of corrupted
gradients relative to clean gradients after subset-specific centering. It is
distinct from the globally centered additive decomposition
\[
S=S_C^{\mathrm{glob}}+S_N^{\mathrm{glob}}
\]
used in Theorem~\ref{thm:local-rank-inflation}.

This distinction is important because globally centered group scatters can
contain between-group mean offsets, whereas \(\widetilde S_C\) and
\(\widetilde S_N\) remove the clean and corrupted group means before effective
rank is computed. Therefore \(NIR\) compares the intrinsic spectral spread of
the clean and corrupted gradient clouds rather than their separation in mean.

Thus Theorem~\ref{thm:local-rank-inflation} and
Proposition~\ref{prop:majorization-nir} address related but logically distinct
quantities: Theorem~\ref{thm:local-rank-inflation} concerns globally centered
additive rank inflation, whereas \(NIR\) measures subset-centered within-group
spectral dispersion.

\begin{proposition}[Majorization implies \(NIR\ge 1\)]
\label{prop:majorization-nir}
Let \(p(\widetilde S_N)\) and \(p(\widetilde S_C)\) be the normalized positive
spectra of the subset-centered corrupted and clean scatter matrices, padded with
zeros if necessary so that they have the same length. We use the convention
\(0\log 0 = 0\). If
\[
p(\widetilde S_N)\prec p(\widetilde S_C),
\]
then
\[
\operatorname{er}(\widetilde S_N)
\ge
\operatorname{er}(\widetilde S_C).
\]
Consequently,
\[
NIR
=
\frac{\operatorname{er}(\widetilde S_N)}
{\operatorname{er}(\widetilde S_C)}
\ge 1.
\]
If the majorization is strict and the spectra are not permutations of one another,
then \(NIR>1\).
\end{proposition}

We emphasize that \(NIR\) is an auxiliary diagnostic that quantifies relative
within-group spectral dispersion after subset-specific centering. Unlike the
Fisher Rank Inflation analyzed in Theorem~\ref{thm:local-rank-inflation},
\(NIR\) does not arise from the globally centered additive decomposition of the
full gradient scatter matrix.

\subsection{Post-Memorization Weakening of the LOO Signal}

Effective rank is scale-invariant:
\[
\operatorname{er}(cS)=\operatorname{er}(S),
\qquad c>0.
\]
Therefore, absolute gradient shrinkage alone cannot imply \(\Delta_i\to 0\).
The correct condition must be stated in terms of normalized spectra.

Here \(\|\cdot\|_1\) denotes the trace norm.

\begin{theorem}[Normalized spectral stability implies LOO collapse]
\label{thm:loo-collapse}
As training progresses through epochs $t$, let $S^{(t)}$ and $S_{-i}^{(t)}$
denote the corresponding full and leave one out centered scatter matrices. Assume that $S^{(t)}$ and $S_{-i}^{(t)}$ have positive trace for all $t$ and $i$
under consideration.
Define
\[
\rho^{(t)}
=
\frac{S^{(t)}}{\operatorname{tr}(S^{(t)})},
\qquad
\rho_{-i}^{(t)}
=
\frac{S_{-i}^{(t)}}{\operatorname{tr}(S_{-i}^{(t)})}.
\]
If
\[
\max_i
\left\|
\rho^{(t)}-\rho_{-i}^{(t)}
\right\|_1
\to 0
\qquad
\text{as } t\to\infty,
\]
then
\[
\max_i
\left|
\Delta_i^{(t)}
\right|
\to 0
\qquad
\text{as } t\to\infty.
\]

\paragraph{Large sample interpretation.}
Let $n$ denote the number of examples used to construct the scatter matrix.
Since
\[
\operatorname{tr}(S^{(t)})
=
\sum_{j=1}^{n}
\|\bar g_j^{(t)}\|^2,
\]
we have
\[
\max_i
\frac{\|\bar g_i^{(t)}\|^2}
{\operatorname{tr}(S^{(t)})}
\geq
\frac{1}{n}.
\]
Thus, this ratio cannot vanish as $t\to\infty$ when $n$ is fixed.
Instead, suppose that
\[
\max_i
\frac{\|\bar g_i^{(t)}\|^2}
{\operatorname{tr}(S^{(t)})}
=
O(1/n)
\]
uniformly over the epochs under consideration. This condition means that no
individual example contributes substantially more than the average scatter
mass. Then
\[
\max_i
\left\|
\rho^{(t)}-\rho_{-i}^{(t)}
\right\|_1
=
O(1/n).
\]
By continuity of entropy, the corresponding leave one out effective rank
contributions become small as $n$ grows.
\end{theorem}

\subsection{Interpretation}

The theory gives a spectral mechanism for Fisher-rank inflation under label noise.

First, Lemma~\ref{lem:effective-rank-derivative} identifies which directions increase
effective rank. Adding spectral mass to low-energy directions increases spectral
entropy. Theorem~\ref{thm:local-rank-inflation} then shows that corrupted-gradient
covariance inflates Fisher effective rank whenever it either creates new spectral
directions outside the clean-gradient span or places sufficient energy in existing
rank-increasing directions.

Second, Lemma~\ref{lem:first-order-loo} shows that the leave-one-out contribution
of a sample is governed, to first order, by a quadratic spectral attribution score
\(\bar g_i^\top B_S\bar g_i\). Theorem~\ref{thm:corrupted-larger-loo} identifies the spectral condition under
which corrupted samples have larger expected LOO rank contribution. The
corrupted-clean covariance difference must have positive net alignment with the
effective-rank sensitivity matrix \(B_S\).
Corollary~\ref{cor:rank-increasing-subspace} gives a simpler
sufficient condition in terms of extra corrupted-gradient energy in rank-increasing
eigendirections.

Third, Proposition~\ref{prop:majorization-nir} explains the noisy-to-clean
effective-rank diagnostic used in the experiments. For NIR, we compute the
effective ranks of the clean and corrupted gradient submatrices after centering
each subset separately. This diagnostic measures within-group spectral dispersion
and is distinct from the additive globally centered decomposition used in
Theorem~\ref{thm:local-rank-inflation}. When the subset-centered noisy spectrum
is more diffuse in the majorization sense, the noisy-to-clean rank ratio satisfies
\(NIR\ge 1\), and it is strictly larger than one under strict spectral spreading.

Finally, Theorem~\ref{thm:loo-collapse} explains why the LOO signal weakens after
memorization. Since effective rank is scale-invariant, the relevant condition is not
raw gradient shrinkage but normalized spectral stability. Removing any one example
must have vanishing effect on the normalized Fisher-gradient spectrum. This matches
the empirical observation that the top noisy fraction among high-LOO contributors
is largest during peak Fisher-rank inflation and decreases later in training.

The theory should be interpreted as a checkpoint-level spectral mechanism rather than a claim about the entire training trajectory.
The experiments below test whether these spectral conditions and attribution
patterns arise at the peak-rank checkpoints observed during noisy-label
training.

\section{Experimental Evaluation}

We investigate the empirical behavior of Fisher Rank Inflation across datasets,
architectures, and corruption levels. Our experiments focus on four questions:

\begin{enumerate}
    \item Does label corruption induce a characteristic inflation--collapse trajectory in Fisher effective rank during training?
    
    \item How does the magnitude of rank inflation vary with the corruption level?
    
    \item Which training examples contribute most strongly to Fisher Rank Inflation?
    
    \item Do the observed dynamics generalize across architectures and datasets?
\end{enumerate}

\subsection{Experimental Setup}

\paragraph{Datasets.}
We evaluate Fisher Rank Inflation on CIFAR-10 and CIFAR-100 under synthetic
symmetric label corruption. CIFAR-10 contains 50{,}000 training images and
10{,}000 test images across 10 classes, while CIFAR-100 contains the same
number of images distributed across 100 classes. To evaluate whether the
phenomenon extends beyond synthetic corruption, we additionally study
CIFAR-10N, which provides naturally occurring human annotation errors. In
particular, we use the \texttt{worse\_label} split for the experiments in
Section~4.6, while retaining the original CIFAR-10 labels as clean references
for evaluation and attribution analysis.

\paragraph{Label Corruption.}
We consider symmetric label corruption. For a corruption rate \(\rho\), each
training label is independently replaced with a uniformly sampled incorrect
class with probability \(\rho\). Unless otherwise specified, we use
\(\rho=0.5\) (50\% corruption). To study the effect of corruption severity, we
additionally perform noise-rate sweeps over
\[
\rho \in \{0,0.2,0.3,0.4,0.5,0.6\}.
\]

\paragraph{Architectures.}
Experiments are conducted using SmallCNN, ResNet18, and Vision Transformers
(ViT) on CIFAR-10. To evaluate behavior on a more challenging dataset, we
additionally train ResNet18 on CIFAR-100. For CIFAR-10N, we use the same CIFAR-adapted ResNet18 architecture as in the CIFAR-10 synthetic-noise experiments.

\paragraph{Fisher-gradient scatter.}
At each epoch, we compute per-example gradients with respect to the final affine
classification layer, including both weights and biases. Equivalently, for each
example we use
\[
g_i
=
(\hat p_i-e_{y_i})\otimes \tilde z_i,
\]
where \(\tilde z_i=[z_i^\top,1]^\top\) is the augmented penultimate
representation. We then construct the centered Fisher-gradient scatter matrix
\[
S
=
\sum_{i=1}^{n}\bar g_i\bar g_i^\top,
\qquad
\bar g_i
=
g_i-\frac{1}{n}\sum_{j=1}^{n}g_j.
\]
We track the entropy effective rank \(\operatorname{er}(S)\) throughout training
and identify the epoch at which effective rank attains its maximum value. We
refer to this checkpoint as the \emph{peak-rank epoch}.

Throughout the experiments, ``Fisher effective rank'' refers to the entropy
effective rank of the centered per-example final-layer gradient scatter.
Thus, the reported spectral quantities measure the dispersion of example-wise
gradient fluctuations around their empirical mean. We use the term
Fisher-gradient scatter to emphasize that the matrix is built from loss
gradients, while reserving the conventional empirical Fisher for the
corresponding uncentered second moment.

\paragraph{Noisy-to-clean inflation ratio.}
To compare the spectral dispersion of clean and corrupted gradients, we compute
\[
\mathrm{NIR}
=
\frac{\operatorname{er}(\widetilde S_N)}
     {\operatorname{er}(\widetilde S_C)},
\]
where \(\widetilde S_N\) and \(\widetilde S_C\) denote the subset-centered
corrupted and clean gradient scatter matrices, respectively. That is, clean and
corrupted gradients are centered separately before their effective ranks are
computed.

\paragraph{Leave-one-out attribution.}
For each training example, we compute the leave-one-out rank contribution
\[
\Delta_i
=
\operatorname{er}(S)-\operatorname{er}(S_{-i}).
\]
Here \(S_{-i}\) is obtained by removing example \(i\), re-centering the remaining
gradients, and recomputing the scatter matrix. Unless otherwise stated,
leave-one-out analyses are performed at the peak-rank epoch, where Fisher Rank
Inflation is strongest.

\paragraph{Evaluation metrics.}
We report peak effective rank, peak NIR, rank-onset and overfitting-onset epochs,
lead time between rank inflation and test degradation, top-100 noisy fraction,
enrichment ratio, AUROC, and AUPRC. For timing comparisons, rank onset is defined as the first evaluated epoch at which Fisher effective rank reaches 20 percent of its maximum increase over the completed training run. This retrospective threshold is used only to compare the timing of rank inflation and test degradation.

\paragraph{Subsampling protocol.}
All experiments use a fixed training subset of 20,000 examples sampled without
replacement from the original training set using the corresponding random seed.
Unless otherwise stated, Fisher effective-rank trajectories are computed on a
deterministic non-augmented Fisher-evaluation subset of 4,096 training examples
for CIFAR-10 experiments and 2,048 examples for CIFAR-100 experiments. Exact
leave-one-out attribution is computed on a deterministic subset of 2,048
examples at the seed-specific peak-rank checkpoint. The same subset ordering is
used when recomputing gradient matrices for direct spectral diagnostics, so the
leave-one-out contributions and diagnostic quantities are evaluated on matched
examples. For CIFAR-10N, we use the same 20{,}000-example training subset size as in
the CIFAR-10 ResNet18 experiments. Fisher trajectories are computed on the
first 4{,}096 examples of the deterministic evaluation ordering, and exact
leave-one-out attribution is computed on the first 2{,}048 examples of the
same ordering at the seed-specific peak-rank checkpoint.

\subsection{Fisher Rank Inflation During Memorization}

We first examine the evolution of the Fisher-gradient spectrum throughout training under symmetric label corruption. Figure~\ref{fig:dynamics} shows a representative training trajectory for ResNet18 trained on CIFAR-10 with \(50\%\) label corruption, including the effective rank of the centered Fisher-gradient scatter, the noisy-to-clean inflation ratio (NIR), and classification performance.

A clear inflation--collapse trajectory emerges. During the early stages of training, the effective rank increases steadily, indicating that gradient covariance becomes distributed across a progressively larger set of spectral directions. The expansion continues until a peak-rank epoch is reached, after which effective rank declines despite continued optimization. We refer to this transient expansion as \emph{Fisher Rank Inflation}.

The evolution of NIR provides complementary evidence for the same spectral transition. As training progresses, the corrupted-gradient scatter becomes more spectrally diffuse relative to the clean-gradient scatter, causing NIR to rise above one during the inflation phase. Near the high-rank portion of training, NIR remains elevated, indicating greater spectral dispersion among corrupted gradients than among clean gradients. This observation is consistent with the theoretical prediction that corrupted labels inject covariance mass into low-energy or previously underutilized eigendirections.

The inflation phase coincides with the transition from structure learning to memorization. Early in training, gradients are concentrated along a relatively small number of dominant directions associated with shared structure in the data. As corrupted labels begin to be fit, gradient covariance spreads into weaker directions, increasing both spectral entropy and effective rank. After memorization, the normalized gradient spectrum becomes progressively more stable, leading to a reduction in effective rank and the observed collapse phase.

In this representative run, rank inflation begins before observable test degradation and reaches its maximum near the transition from structure learning to memorization. This shows that the identified onset of Fisher Rank Inflation can occur before observable test degradation. The magnitude of this lead time varies across architectures and datasets, so we interpret it as a retrospective timing comparison rather than an online prediction rule.

Appendix~\ref{app:onset-threshold-sensitivity} shows that this lead-time
pattern is robust to a grid of rank-onset and test-drop thresholds for the
convolutional models, while remaining more threshold-dependent for the Vision
Transformer.

Taken together, these results identify Fisher Rank Inflation as a spectral signature of memorization under label noise. The existence of a distinct peak-rank epoch further motivates the sample-level analyses presented in the following subsection.

\begin{figure}[H]
    \centering
    \includegraphics[width=1\linewidth]{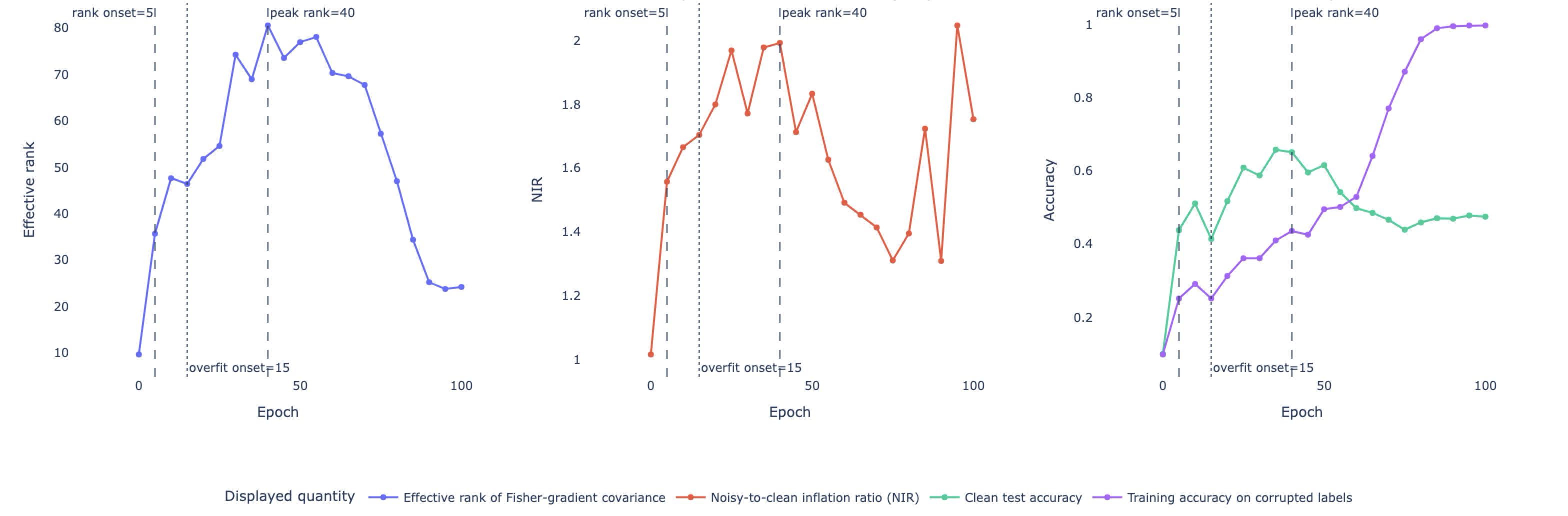}
    \caption{Representative training trajectory showing Fisher Rank Inflation for ResNet18 on CIFAR-10 with \(50\%\) symmetric label corruption. The effective rank of the Fisher-gradient scatter exhibits a characteristic inflation--collapse trajectory in this run, peaking at epoch 40. The noisy-to-clean inflation ratio (NIR) increases during the inflation phase, indicating growing spectral separation between clean and corrupted gradients. Rank inflation begins before observable test degradation and reaches its maximum near the transition from structure learning to memorization.
}
    \label{fig:dynamics}
\end{figure}

\subsection{Corrupted Examples Drive Fisher Rank Inflation}

The inflation--collapse dynamics observed in the previous subsection establish
that Fisher effective rank undergoes a characteristic expansion during
memorization. We now investigate which training examples are responsible for
this expansion.

To quantify the contribution of individual samples, we compute the leave-one-out
rank contribution
\[
\Delta_i
=
\operatorname{er}(S)-\operatorname{er}(S_{-i}),
\]
which measures the reduction in effective rank caused by removing example
\(i\). Positive values of \(\Delta_i\) indicate that the corresponding sample
supports the inflated Fisher spectrum, whereas values near zero indicate a
negligible contribution.

Figure~\ref{fig:loo_resnet18} shows a representative peak-rank leave-one-out
analysis for ResNet18 trained on CIFAR-10 with \(50\%\) symmetric label
corruption. A pronounced separation emerges between clean and corrupted
examples. The distribution of corrupted-example contributions is shifted toward
larger positive values and exhibits substantially greater variance. In contrast,
clean examples are concentrated near zero contribution.

This difference is reflected in both mean and median attribution scores in the
representative run. Corrupted examples achieve a mean contribution of
\(0.0125\), compared to \(-0.0056\) for clean examples. Similarly, the median
contribution increases from \(-0.0047\) for clean examples to \(0.0110\) for
corrupted examples. Thus, removing a corrupted example typically produces a
larger reduction in effective rank than removing a clean example.

To further quantify this effect, we rank all evaluated training examples
according to their leave-one-out contribution. In this representative run,
corrupted examples constitute \(49.4\%\) of the evaluated subset but account for
\(97\%\) of the top-100 rank-contributing samples. This corresponds to an
enrichment factor of approximately \(1.96\times\) relative to the background
corruption rate. Across five random seeds, the same ResNet18 CIFAR-10 setting
achieves a top-100 noisy fraction of \(95.0\%\pm1.4\%\) and an enrichment factor
of \(1.904\pm0.037\), as reported in Table~\ref{tab:cross_arch}. Thus, the
highest-contributing examples are consistently dominated by corrupted samples.

These findings provide sample-level evidence consistent with the spectral
mechanism developed in Section~\ref{sec:theory}. The theory identifies spectral conditions under which examples aligned with rank-increasing directions contribute disproportionately to Fisher Rank Inflation. Consistent with
this mechanism, corrupted examples exhibit systematically larger
leave-one-out rank contributions than clean examples and dominate the
highest-contributing samples across seeds. Fisher Rank Inflation is therefore not a uniform property of the training set. Instead, it is driven primarily by a subset of corrupted examples whose gradients inject variance into rank-expanding directions of the Fisher-gradient scatter.

\begin{figure}[H]
    \centering
    \begin{subfigure}{0.48\textwidth}
        \centering
        \includegraphics[width=\linewidth]{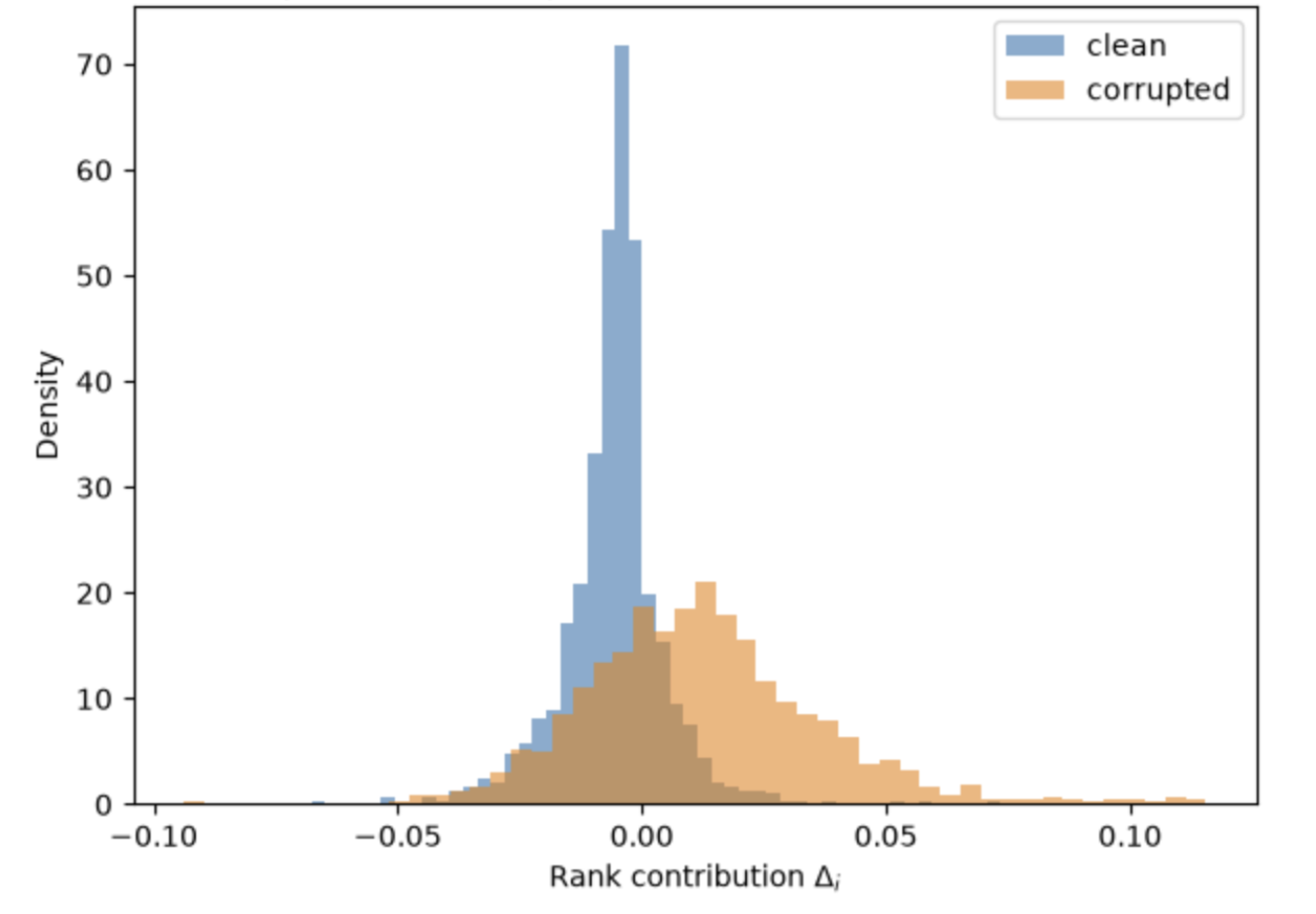}
    \end{subfigure}
    \hfill
    \begin{subfigure}{0.48\textwidth}
        \centering
        \includegraphics[width=\linewidth]{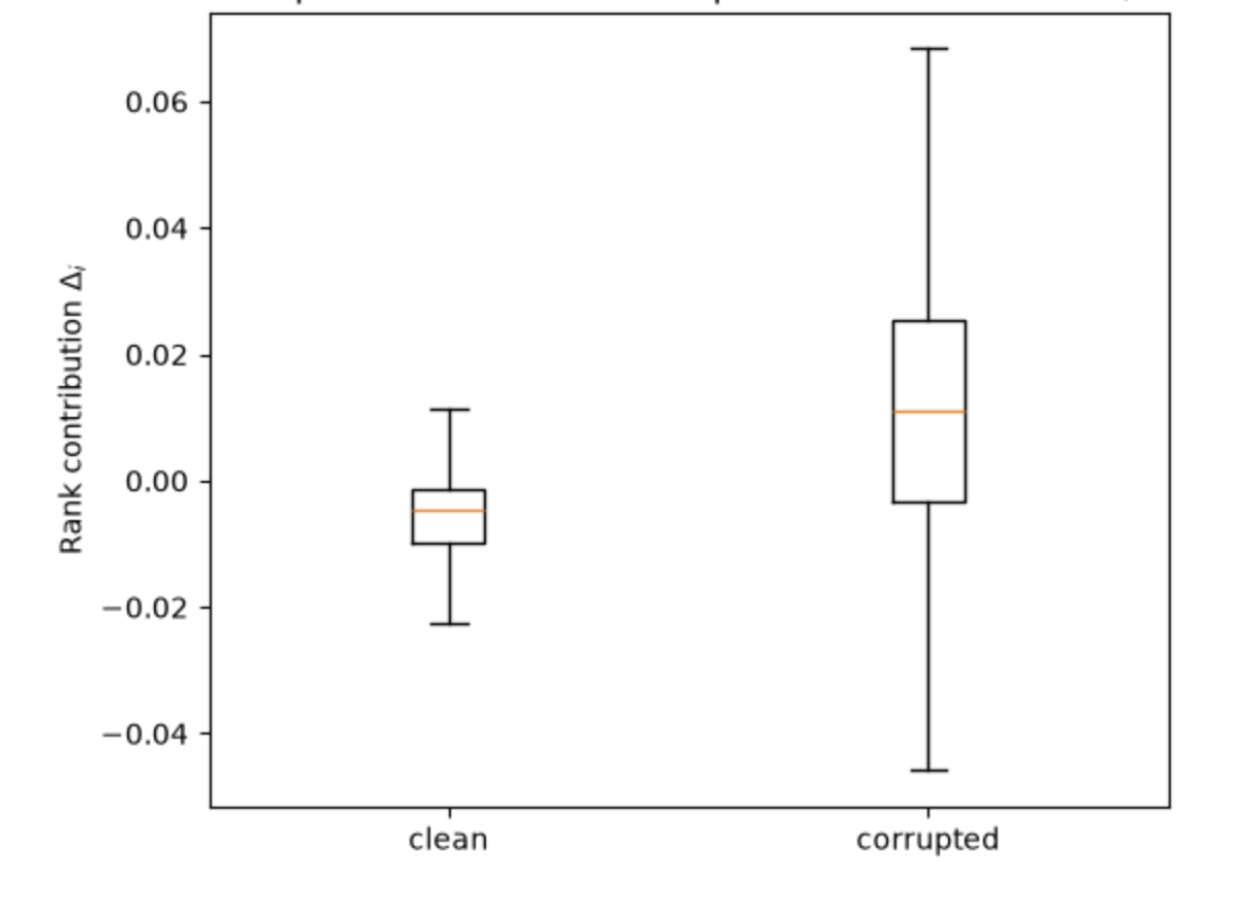}
    \end{subfigure}

    \caption{Representative leave-one-out rank contributions at the peak-rank checkpoint for ResNet18 on CIFAR-10 with \(50\%\) symmetric label corruption. Corrupted examples exhibit substantially larger rank contributions than clean examples, indicating that they disproportionately support the inflated Fisher-gradient spectrum. In this representative run, corrupted samples constitute roughly half of the evaluated training subset but account for \(97\%\) of the top-100 rank-contributing examples. Across five seeds, the corresponding top-100 noisy fraction is \(95.0\%\pm1.4\%\), as reported in Table~\ref{tab:cross_arch}.}

    \label{fig:loo_resnet18}
\end{figure}

\subsection{Clean-Difficulty Control}
\label{sec:clean_difficulty}

A natural alternative explanation for the leave-one-out Fisher-rank attribution results
is that they merely identify difficult training examples rather than corrupted labels.
To distinguish optimization difficulty from label corruption, we perform a controlled
experiment that explicitly separates these two factors.

Using ResNet18 on CIFAR-10 with 50\% symmetric label corruption, we partition the
training samples into three disjoint groups:

\begin{enumerate}
    \item \textbf{Normal clean}: correctly labeled examples with relatively small
    training loss.
    \item \textbf{High-loss clean}: correctly labeled examples having the largest
    training losses among all correctly labeled samples. These samples serve as a
    proxy for rare, ambiguous, or intrinsically difficult examples that remain
    correctly labeled despite being challenging to optimize.
    \item \textbf{Memorized corrupted}: examples whose observed training labels are
    incorrect but whose training losses with respect to those incorrect labels are
    already very small. These samples have therefore been successfully memorized by
    the network despite receiving incorrect supervision.
\end{enumerate}

The high-loss clean group controls for optimization difficulty, while the
memorized corrupted group isolates label corruption after memorization. If
Fisher Rank Inflation primarily reflected sample difficulty, the high-loss
clean group would be expected to exhibit the largest leave-one-out
Fisher-rank contributions. Conversely, if Fisher Rank Inflation is driven by
memorized label corruption, memorized corrupted examples should dominate
despite their low training loss. Figure~\ref{fig:clean_difficulty_boxplot}
provides a visual comparison of the leave-one-out Fisher-rank contribution
distributions for the three groups.

Table~\ref{tab:clean_difficulty_control} supports the latter hypothesis.
Memorized corrupted examples exhibit substantially larger positive leave-one-out
contributions than either category of clean examples, whereas high-loss clean
examples have a slightly negative mean contribution. Moreover, memorized corrupted
examples account for approximately $96\%$ of the top-100 Fisher-rank contributors
across three independent random seeds.

These results demonstrate that Fisher Rank Inflation is not simply identifying
difficult or rare training examples. Instead, it preferentially identifies
examples whose incorrect labels have already been memorized, providing additional
evidence that the observed Fisher-spectrum expansion is associated with memorized
label corruption rather than optimization difficulty alone.

\begin{table}[t]
\centering
\caption{Clean-difficulty control experiment averaged over three random seeds.
Results are evaluated at the seed-specific peak Fisher-rank checkpoint for
ResNet18 on CIFAR-10 with 50\% symmetric label corruption.}
\label{tab:clean_difficulty_control}
\begin{tabular}{lcc}
\toprule
Group & Mean $\Delta_i$ & Median $\Delta_i$\\
\midrule
High-loss clean & $-0.0137 \pm 0.0012$ & $0.00006 \pm 0.00003$\\
Normal clean & $0.0046 \pm 0.0005$ & $0.00006 \pm 0.00003$\\
Memorized corrupted & $\mathbf{0.0751 \pm 0.0027}$ & $\mathbf{0.0707 \pm 0.0058}$\\
\bottomrule
\end{tabular}
\end{table}

\begin{figure}[H]
    \centering
    \includegraphics[width=0.65\linewidth]{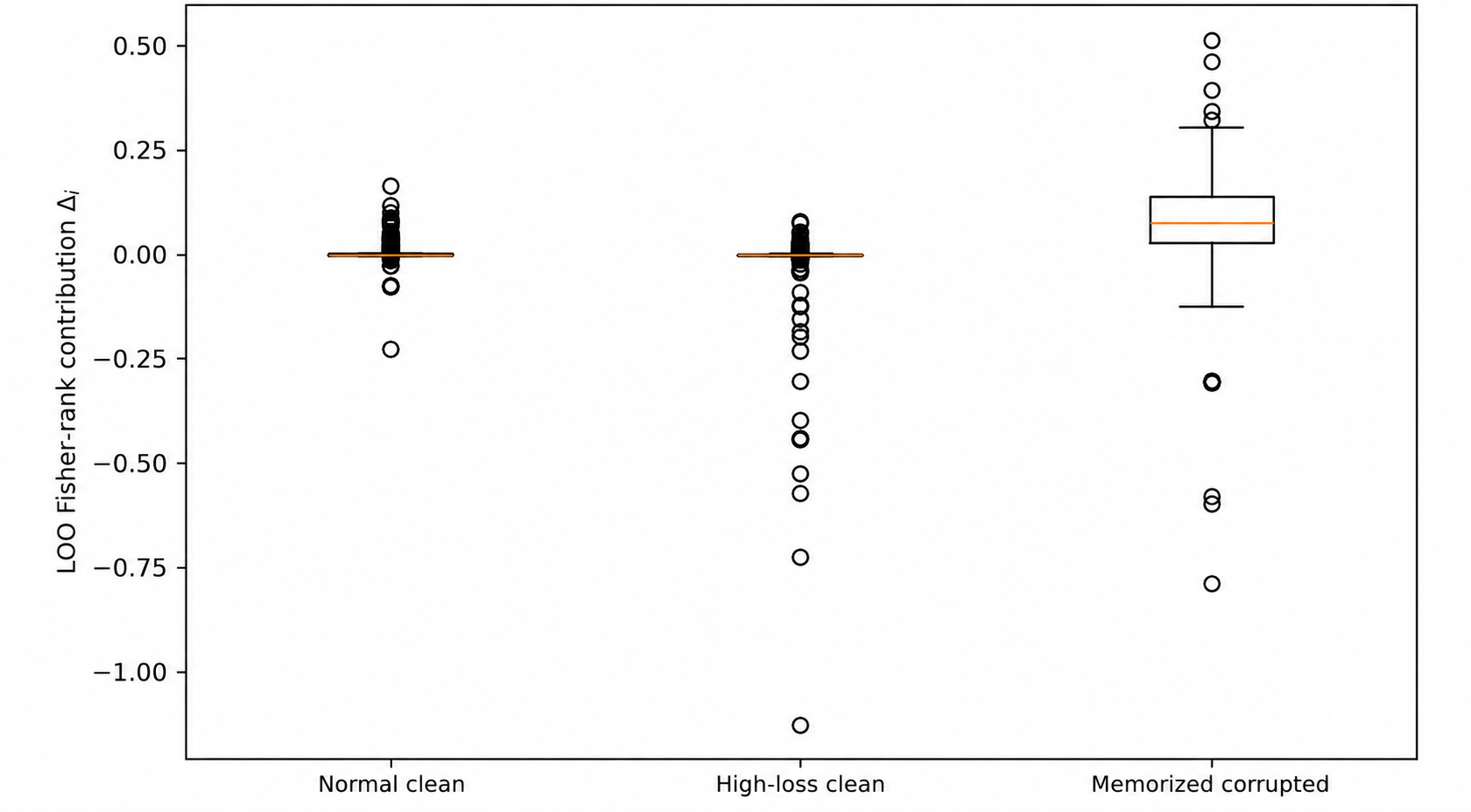}
    \caption{
Clean-difficulty control experiment for ResNet18 trained on CIFAR-10 with
50\% symmetric label corruption. Leave-one-out Fisher-rank contributions
($\Delta_i$) are shown for normal clean examples, high-loss clean examples,
and memorized corrupted examples at the seed-specific peak Fisher-rank
checkpoint. Although high-loss clean examples are difficult to optimize,
their contributions remain near zero, whereas memorized corrupted examples
exhibit substantially larger positive contributions. This indicates that
Fisher Rank Inflation preferentially identifies memorized label corruption
rather than optimization difficulty.
}
    \label{fig:clean_difficulty_boxplot}
\end{figure}

\subsection{Cross-Architecture and Cross-Dataset Validation}
\label{sec:cross-arch}

We next investigate whether the attribution patterns observed for ResNet18 on
CIFAR-10 generalize across architectures and datasets. Table~\ref{tab:cross_arch}
summarizes leave-one-out attribution statistics at the seed-specific peak-rank
checkpoint for all experimental settings considered in this work. For each
setting, we report mean and standard deviation across five random seeds.

Across all settings, corrupted examples are enriched among the highest
rank-contributing samples. Although corrupted examples constitute approximately
\(50\%\) of the training subset, they account for substantially more than half
of the top-100 rank-contributing examples in every architecture and dataset.
The top-100 noisy fraction ranges from \(69.2\%\) for ViT on CIFAR-10 to
\(96.2\%\) for SmallCNN on CIFAR-10. This enrichment persists across both
convolutional and transformer architectures, indicating that Fisher Rank
Inflation is not tied to a particular model family.

The strongest enrichment is observed for the convolutional CIFAR-10 models.
For SmallCNN, corrupted examples account for \(96.2\%\pm2.0\%\) of the top-100
rank-contributing samples, corresponding to an enrichment factor of
\(1.928\pm0.064\) relative to the background corruption rate. ResNet18 on
CIFAR-10 exhibits similarly strong enrichment. Corrupted examples account for
\(95.0\%\pm1.4\%\) of the top-100 rank-contributing samples, with enrichment
\(1.904\pm0.037\).

The effect is weaker but still present in the Vision Transformer. For ViT on
CIFAR-10, corrupted examples account for \(69.2\%\pm3.7\%\) of the top-100
rank-contributing samples, corresponding to an enrichment factor of
\(1.387\pm0.087\). Thus, while the attribution signal is attenuated in the
transformer architecture, the highest rank-contributing examples remain
enriched for corrupted labels relative to the background corruption rate.

The effect also extends to the more challenging CIFAR-100 setting. For
ResNet18 on CIFAR-100, corrupted examples account for \(92.2\%\pm1.6\%\) of
the top-100 rank-contributing samples, corresponding to an enrichment factor of
\(1.848\pm0.052\). This indicates that Fisher-rank attribution remains
effective even when the number of classes and the final-layer gradient
dimension are substantially larger.

Rank-contribution attribution provides above-chance, but architecture-dependent, separation between clean and corrupted examples. AUROC values range from
\(0.601\pm0.024\) for ViT on CIFAR-10 to \(0.806\pm0.011\) for ResNet18 on
CIFAR-100, while AUPRC values range from \(0.601\pm0.025\) to
\(0.798\pm0.010\). These values indicate that leave-one-out rank contribution
captures information about label corruption beyond a single architecture or
dataset. At the same time, we emphasize that rank contribution is intended as a
spectral attribution measure rather than a generic noisy-label detector. Loss
and confidence baselines can yield stronger global AUROC in some settings.

Taken together, these findings show that the relationship between corrupted examples and Fisher Rank
Inflation extends across multiple architectures and datasets. The consistent enrichment of corrupted examples
among the largest rank contributors supports the view that Fisher Rank Inflation is not an artifact of a
single experimental setting, although the strength of the attribution signal varies across model families.

\begin{table}[H]
\centering
\caption{Cross-architecture and cross-dataset attribution results at the peak-rank checkpoint.}
\label{tab:cross_arch}
\begin{tabular}{llccccc}
\toprule
Dataset & Model & Peak Epoch & Top-100 Noisy & Enrichment & AUROC & AUPRC \\
\midrule
CIFAR-10 & SmallCNN & 66.0±4.2  & 0.962±0.020 & 1.928±0.064 & 0.683±0.010 & 0.757±0.008 \\
CIFAR-10 & ResNet18 & 52.0±2.7  & 0.950±0.014 & 1.904±0.037 & 0.681±0.024 & 0.762±0.023 \\
CIFAR-10 & ViT      & 101.0±6.5 & 0.692±0.037 & 1.387±0.087 & 0.601±0.024 & 0.601±0.025 \\
CIFAR-100 & ResNet18 & 24.0±4.2 & 0.922±0.016 & 1.848±0.052 & 0.806±0.011 & 0.798±0.010 \\
\bottomrule
\end{tabular}
\end{table}

\paragraph{Direct spectral diagnostics.}
To test the spectral mechanism more directly, we evaluate the checkpoint-level
quantities appearing in Section~3 at the seed-specific peak-rank checkpoints.
For each setting, we compute the fraction of corrupted-gradient scatter outside
the globally centered clean-gradient span, the first-order sensitivity gap
\(\mathbb{E}[A_i(S)\mid i\in N]-\mathbb{E}[A_i(S)\mid i\in C]\), the
correlation between the first-order attribution score \(A_i(S)\) and the exact
leave-one-out contribution \(\Delta_i\), and the corrupted-example enrichment
among the top exact leave-one-out contributors.

\begin{table}[H]
\centering
\small
\caption{Direct diagnostics for the spectral quantities in Section~3,
evaluated at seed-specific peak-rank checkpoints under 50\% symmetric label
corruption. The new-direction fraction is
\(\operatorname{tr}(Q_C S_N^{\mathrm{glob}}Q_C)/
\operatorname{tr}(S_N^{\mathrm{glob}})\). The sensitivity gap is
\(\mathbb{E}[A_i(S)\mid i\in N]-\mathbb{E}[A_i(S)\mid i\in C]\), where
\(A_i(S)=\bar g_i^\top B_S\bar g_i\). Results are reported as mean
\(\pm\) standard deviation over five seeds.}
\label{tab:direct-spectral-diagnostics}
\begin{tabular}{llccccc}
\toprule
Dataset & Model
& New-dir. frac.
& Sens. gap
& Spearman \(A_i,\Delta_i\)
& \(\Delta_N-\Delta_C\)
& Top-100 noisy \(\Delta_i\) \\
\midrule
CIFAR-10 & SmallCNN
& \(0.0000\pm0.0000\)
& \(0.0058\pm0.0006\)
& \(0.9795\pm0.0038\)
& \(0.0082\pm0.0008\)
& \(0.962\pm0.020\) \\

CIFAR-10 & ResNet18
& \(0.0799\pm0.0062\)
& \(0.0079\pm0.0016\)
& \(0.9603\pm0.0081\)
& \(0.0136\pm0.0019\)
& \(0.950\pm0.014\) \\

CIFAR-10 & ViT
& \(0.1450\pm0.0080\)
& \(-0.0011\pm0.0011\)
& \(0.7036\pm0.0270\)
& \(0.0340\pm0.0037\)
& \(0.694\pm0.040\) \\

CIFAR-100 & ResNet18
& \(0.5879\pm0.0117\)
& \(0.0379\pm0.0223\)
& \(0.6514\pm0.0999\)
& \(0.1668\pm0.0118\)
& \(0.922\pm0.016\) \\
\bottomrule
\end{tabular}
\end{table}

The direct diagnostics provide checkpoint-level support for the spectral
mechanism, while also showing that different architectures realize the
mechanism in different ways. For CIFAR-10 SmallCNN and ResNet18, the
first-order attribution score \(A_i(S)=\bar g_i^\top B_S\bar g_i\) is highly
aligned with the exact leave-one-out contribution \(\Delta_i\), with Spearman
correlations above \(0.96\). In both cases, the corrupted group has a positive
sensitivity gap, supporting the sufficient condition in Theorem~2.

The two convolutional CIFAR-10 models differ in the role of new spectral
directions. For ResNet18, corrupted gradients place a nonzero fraction of their
globally centered scatter outside the clean-gradient span. For SmallCNN, the
new-direction fraction is numerically zero, while the sensitivity gap remains
positive and \(A_i(S)\) is strongly correlated with exact leave-one-out
attribution. This suggests that SmallCNN rank inflation is primarily explained
by corrupted-gradient covariance within existing rank-increasing directions,
corresponding to the in-span case of Theorem~1.

For CIFAR-100 ResNet18, corrupted gradients place a substantially larger
fraction of their scatter outside the clean-gradient span, and exact
leave-one-out attribution remains strongly enriched for corrupted examples.
The correlation between \(A_i(S)\) and \(\Delta_i\) is more moderate than in
CIFAR-10, suggesting that higher-order deletion effects are more pronounced in
the higher-cardinality output setting.

For the Vision Transformer, exact leave-one-out attribution still shows
corrupted-example enrichment, but the first-order sensitivity gap is slightly
negative. Thus, the sufficient condition in Theorem~2 is not uniformly
satisfied in the transformer setting. Since Theorem~2 gives a sufficient
rather than necessary condition, this does not contradict the theory. Instead,
it matches the weaker ViT attribution signal reported in Table~\ref{tab:cross_arch}.

\subsection{Rank Inflation Increases with Corruption Severity}
\label{sec:noise-sweep}

Section~\ref{sec:theory} suggests that Fisher Rank Inflation should become more pronounced as the amount of label corruption increases. Corrupted examples introduce gradient variability that can populate weak or previously underutilized eigendirections of the Fisher-gradient scatter, thereby increasing spectral entropy and effective rank. As the fraction of corrupted labels grows, a larger number of examples are expected to contribute to this rank-expanding mechanism.

To test this prediction, we perform a label-noise sweep on CIFAR-10 using ResNet18 under symmetric corruption rates
\[
\rho \in \{0.0,0.2,0.3,0.4,0.5,0.6\}.
\]
For each corruption level, we train three random seeds and record the maximum Fisher effective rank attained during training. We also compute the peak noisy-to-clean inflation ratio (NIR), which measures the relative spectral dispersion of corrupted gradients compared to clean gradients.

Figure~\ref{fig:noise-sweep} shows that the magnitude of Fisher Rank Inflation increases systematically with corruption severity. Under clean training, the peak Fisher effective rank is \(28.88 \pm 1.95\). As label corruption increases, the peak effective rank rises to \(62.69 \pm 2.32\), \(70.37 \pm 2.48\), \(78.98 \pm 0.76\), \(87.90 \pm 3.63\), and \(97.09 \pm 1.78\) for corruption rates of \(20\%\), \(30\%\), \(40\%\), \(50\%\), and \(60\%\), respectively. Thus, increasing corruption severity produces a larger expansion of the Fisher-gradient spectrum.

The NIR diagnostic remains consistently above one across noisy settings. Peak NIR values are \(2.00 \pm 0.10\), \(1.97 \pm 0.04\), \(2.01 \pm 0.00\), \(1.98 \pm 0.03\), and \(1.87 \pm 0.09\) for corruption rates of \(20\%\), \(30\%\), \(40\%\), \(50\%\), and \(60\%\), respectively. This indicates that corrupted-gradient subsets remain more spectrally dispersed than clean-gradient subsets throughout the noisy-label regimes. Unlike peak effective rank, NIR is not strictly monotonic in the corruption rate, but it remains elevated above one whenever corrupted labels are present.

These observations are consistent with the proposed mechanism. Increasing the number of corrupted labels increases the number of examples whose gradients inject covariance mass into weak eigendirections of the Fisher-gradient scatter. The resulting increase in spectral entropy produces larger effective-rank expansions and stronger Fisher Rank Inflation.

Taken together, these results provide empirical support for the central
mechanistic implication of our theory. Increasing label corruption introduces
more gradient covariance in rank-expanding directions and thereby produces
larger Fisher-rank expansion. This behavior further supports the view that Fisher Rank Inflation is a consequence of corruption-driven memorization rather than a generic artifact of deep-network training.

\begin{figure}[H]
    \centering
    \includegraphics[width=1\linewidth]{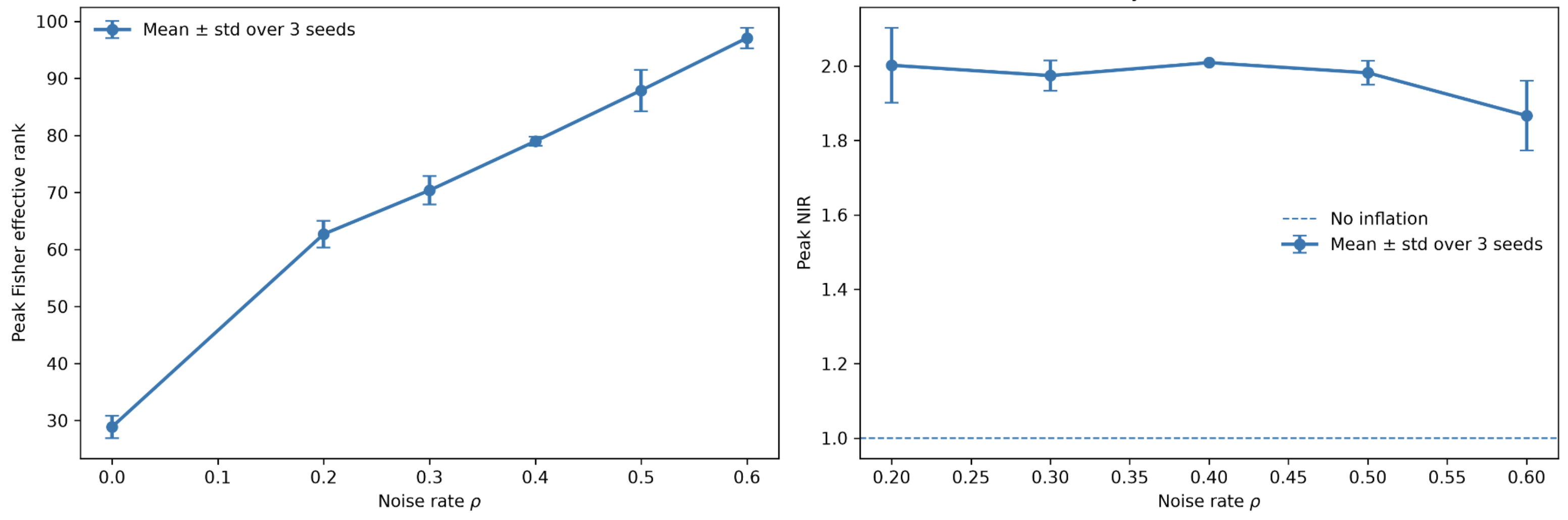}
    \caption{Dependence of Fisher Rank Inflation on corruption severity for ResNet18 trained on CIFAR-10 under symmetric label corruption. Results are averaged over three random seeds, with error bars denoting standard deviation. Peak Fisher effective rank increases monotonically with the corruption rate, rising from \(28.88 \pm 1.95\) under clean training to \(97.09 \pm 1.78\) at \(60\%\) corruption. Peak noisy-to-clean inflation ratio (NIR) remains consistently above one across noisy settings, indicating greater spectral dispersion among corrupted gradients than among clean gradients.}
    \label{fig:noise-sweep}
\end{figure}

\subsection{Fisher Rank Inflation under Human Annotation Noise}

The preceding experiments focus on synthetic symmetric label corruption. To evaluate whether Fisher Rank Inflation extends beyond artificially generated noise, we additionally study CIFAR-10N \citep{wei2022learningnoisylabelsrevisited}, which contains naturally occurring human annotation errors collected from multiple annotators. Unlike synthetic corruption, these label errors arise from human mistakes and ambiguity rather than random label replacement.

We evaluate ResNet18 on the CIFAR-10N \texttt{worse\_label} split using the same training protocol, Fisher-rank computation, and leave-one-out attribution procedure described in Section~4.1. Models are trained using the observed human labels while clean labels are retained for evaluation and attribution analysis. Results are averaged over five random seeds.

\begin{table}[H]
\centering
\caption{Fisher Rank Inflation under human annotation noise on CIFAR-10N. Results are reported as mean
$\pm$ standard deviation over five random seeds. The human-noisy fraction is
$39.7\% \pm 1.0\%$ of the evaluated subset.}
\label{tab:cifar10n}
\begin{tabular}{lccccc}
\toprule
Dataset & Noise Type & Top-100 Noisy & Enrichment & AUROC & AUPRC \\
\midrule
CIFAR-10N & \texttt{worse\_label}
& $0.944 \pm 0.019$
& $2.377 \pm 0.093$
& $0.608 \pm 0.010$
& $0.645 \pm 0.014$ \\
\bottomrule
\end{tabular}
\end{table}

Despite the substantially different source of label noise, human-noisy examples
are strongly enriched among the examples with the largest Fisher-rank
contributions. In particular, they account for $94.4\% \pm 1.9\%$ of the
top-100 rank-contributing samples while constituting only
$39.7\% \pm 1.0\%$ of the evaluated subset, corresponding to an enrichment
factor of $2.377 \pm 0.093$.

This result supports the view that Fisher-rank attribution captures examples that disproportionately support
Fisher-spectrum expansion during memorization. At the same time, rank contribution is not intended as a
generic noisy-label detector. Its global separation performance on CIFAR-10N is modest
(AUROC $0.608 \pm 0.010$, AUPRC $0.645 \pm 0.014$), even though its top-ranked examples are highly
enriched for human annotation errors. In addition, Fisher rank onset precedes the chosen overfitting onset criterion under human annotation noise, with a mean lead time of $29.0 \pm 13.1$ epochs
across five seeds.

These results provide direct evidence that Fisher Rank Inflation is not restricted to synthetic label corruption. The same enrichment phenomenon arises under naturally occurring human annotation errors, suggesting that the spectral mechanism identified in Section~\ref{sec:theory} captures a broader aspect of memorization than the specific corruption process used to generate noisy labels.

\section{Discussion}

Our results suggest that memorization under label noise is accompanied by a distinct spectral reorganization of the Fisher-gradient scatter. Rather than viewing memorization solely through prediction dynamics or loss trajectories, the Fisher-rank perspective shows how corrupted examples change the distribution of gradient variation across directions during training. In particular, Fisher Rank Inflation
emerges when corrupted examples inject covariance mass into weak or previously
underutilized eigendirections, increasing spectral entropy and expanding the
effective rank of the Fisher-gradient scatter.

A key observation of this work is that Fisher Rank Inflation is not driven
uniformly by all training examples. The leave-one-out attribution analysis
shows that corrupted examples contribute disproportionately to effective-rank
expansion and dominate the highest rank-contributing samples across datasets
and architectures. This observation provides a sample-level interpretation of
the inflation phenomenon and links global spectral dynamics to individual
training examples.

Our findings also clarify the relationship between memorization and
overfitting. While these concepts are often discussed interchangeably, they
refer to distinct phenomena. Memorization describes the process by which a
network fits corrupted or idiosyncratic training examples, whereas overfitting
refers to degradation in generalization performance. Fisher Rank Inflation
emerges during memorization and often precedes observable test degradation.
This suggests that reorganization of gradient covariance occurs before
generalization performance visibly deteriorates. At the same time, the Vision
Transformer experiments demonstrate that the magnitude of this lead time can
vary across architectures. Fisher Rank Inflation should therefore be viewed as
a spectral signature of memorization rather than a universal predictor of
future overfitting.

Although our analysis focuses on symmetric label corruption, the underlying
mechanism is not tied to a particular corruption model. The theory predicts
rank inflation whenever corrupted examples inject covariance mass into weak or
previously underrepresented eigendirections of the Fisher-gradient scatter.
Consequently, similar behavior may arise under asymmetric, class-dependent, or
instance-dependent corruption, provided that the resulting gradients introduce
sufficient spectral diversification. Investigating these settings remains an
important direction for future work.

Several limitations should be noted. First, our theoretical results are local checkpoint spectral statements. They identify conditions under which corrupted gradients increase effective rank and receive larger leave-one-out attribution. Second, broader validation across larger models, alternative modalities, and structured or instance-dependent corruption remains necessary. Third, the first-order attribution approximation is highly accurate for the convolutional CIFAR-10 models but weaker for the Vision Transformer and CIFAR-100. This indicates that higher-order deletion effects can be important. Fourth, the clean-difficulty control experiment relies on high-loss correctly labeled examples as a proxy for difficult, rare, or ambiguous clean samples. Training loss is an imperfect measure of intrinsic difficulty and may also reflect transient optimization effects, while some genuinely difficult examples may already have low loss. Finally, although Fisher Rank Inflation often precedes test degradation and strongly enriches corrupted examples among the largest contributors, the present results support it primarily as a spectral signature of memorization rather than as a universally validated online stopping rule or generic noisy-label detector.

Overall, Fisher Rank Inflation provides a new perspective on memorization in
deep networks. By connecting gradient-spectrum dynamics, sample-level
attribution, and label corruption through a common spectral framework, our
results suggest that effective-rank dynamics offer a useful lens for studying
how deep networks transition from learning structure to fitting noise.

\section{Conclusion}

We introduced \emph{Fisher Rank Inflation}, a phenomenon in the gradient spectrum that emerges during memorization under label noise. Across architectures and datasets, we
observed a characteristic inflation--collapse trajectory in the effective rank
of the Fisher-gradient scatter. Effective rank expands as corrupted labels
begin to be fit, reaches its maximum, and subsequently contracts after
memorization.

To explain this behavior, we link effective-rank to the redistribution of covariance mass across eigendirections of the
Fisher-gradient scatter. Corrupted examples introduce variance into weak or previously unused spectral directions. This raises the spectral entropy and, in turn, effective rank.
 We further derived a
leave-one-out attribution framework and identified conditions under which
corrupted examples contribute more strongly to rank inflation than clean
examples.

Empirically, we tested these mechanisms on CIFAR-10 and CIFAR-100 using
SmallCNN, ResNet18, and Vision Transformers under symmetric label corruption.
Fisher Rank Inflation consistently appeared during memorization, increased in
magnitude with corruption severity, and was driven disproportionately by
corrupted examples. At peak-rank checkpoints, corrupted samples
were strongly enriched among the highest rank-contributing examples, with
top-100 noisy fractions ranging from \(69.2\%\) to \(96.2\%\) across five-seed
experiments. Direct checkpoint-level diagnostics further linked the
first-order spectral attribution score to exact leave-one-out rank
contributions, with architecture-dependent strength. A seeded corruption sweep
showed that peak Fisher effective rank increases monotonically with corruption
severity, rising from \(28.88 \pm 1.95\) under clean training to
\(97.09 \pm 1.78\) at \(60\%\) corruption.

Taken together, these results establish Fisher Rank Inflation as a spectral
signature of memorization in deep networks. More broadly, our findings suggest
that effective-rank dynamics provide a useful perspective for studying how gradient
representations evolve during training and how corrupted examples reshape the
spectral structure of learning. Additional experiments on CIFAR-10N demonstrate that Fisher Rank Inflation persists under naturally occurring human annotation errors, indicating that the phenomenon is not an artifact of synthetic symmetric corruption.

\bibliographystyle{abbrv}
\bibliography{references}

\appendix
\section{Appendix}

\subsection{Proofs}
\subsubsection{Proof of Lemma~1}

\begin{proof}
Since \(p_j=\lambda_j/T\), we have
\[
\frac{\partial p_j}{\partial \lambda_k}
=
\frac{\mathbf{1}\{j=k\}-p_j}{T}.
\]
Therefore,
\[
\frac{\partial H}{\partial \lambda_k}
=
-\sum_j(\log p_j+1)\frac{\partial p_j}{\partial \lambda_k}
=
-\frac{1}{T}\bigl(\log p_k+H(S)\bigr).
\]
Applying the chain rule to \(\operatorname{er}(S)=e^{H(S)}\) gives
\[
\frac{\partial \operatorname{er}(S)}{\partial \lambda_k}
=
-\frac{\operatorname{er}(S)}{T}
\bigl(\log p_k+H(S)\bigr).
\]
The sign condition follows because
\[
\log p_k+H(S)<0
\quad \Longleftrightarrow \quad
p_k < e^{-H(S)} = \frac{1}{\operatorname{er}(S)}.
\]
\end{proof}

\subsubsection{Proof of Lemma~2}

\begin{proof}
Let
\[
\mu_{-i}
=
\frac{1}{n-1}\sum_{j\neq i}g_j
=
\frac{n\mu-g_i}{n-1}.
\]
Then
\[
\mu-\mu_{-i}
=
\frac{g_i-\mu}{n-1}
=
\frac{\bar g_i}{n-1}.
\]
Hence, for \(j\neq i\),
\[
g_j-\mu_{-i}
=
(g_j-\mu)+(\mu-\mu_{-i})
=
\bar g_j+\frac{\bar g_i}{n-1}.
\]
Therefore,
\[
S_{-i}
=
\sum_{j\neq i}
\left(
\bar g_j+\frac{\bar g_i}{n-1}
\right)
\left(
\bar g_j+\frac{\bar g_i}{n-1}
\right)^\top.
\]
Expanding and using \(\sum_{j=1}^n \bar g_j=0\), so that
\[
\sum_{j\neq i}\bar g_j = -\bar g_i,
\]
we obtain
\[
S_{-i}
=
S-\bar g_i\bar g_i^\top
-\frac{2}{n-1}\bar g_i\bar g_i^\top
+\frac{1}{n-1}\bar g_i\bar g_i^\top.
\]
Thus
\[
S_{-i}
=
S-\left(1+\frac{1}{n-1}\right)\bar g_i\bar g_i^\top
=
S-\frac{n}{n-1}\bar g_i\bar g_i^\top.
\]
\end{proof}

\subsubsection{Proof of Lemma~3}

\begin{proof}
By Lemma~\ref{lem:centered-deletion},
\[
S_{-i}=S-a\bar g_i\bar g_i^\top,
\qquad
a=\frac{n}{n-1}.
\]
The perturbation has operator norm \(a\|\bar g_i\|^2\). The condition
\[
a\|\bar g_i\|^2\le c\delta(S)
\]
ensures that the positive eigenvalues remain isolated from one another and
separated from zero along the deletion path
\[
S-t a\bar g_i\bar g_i^\top,
\qquad t\in[0,1].
\]
Hence the positive spectral support is locally stable, and standard first-order
eigenvalue perturbation gives
\[
\lambda_k(S_{-i})
=
\lambda_k(S)-a(u_k^\top \bar g_i)^2
+
O(\|\bar g_i\|^4).
\]

This follows from standard perturbation theory for symmetric matrices with
simple isolated eigenvalues. The second-order remainder is controlled by the
local spectral gap \(\delta(S)\).

Taylor expanding effective rank as a smooth function of the positive eigenvalues
inside this local spectral neighborhood gives
\[
\operatorname{er}(S_{-i})
=
\operatorname{er}(S)
-
a\sum_{k=1}^r
\beta_k(S)(u_k^\top \bar g_i)^2
+
O(\|\bar g_i\|^4).
\]
Rearranging proves the claim.
\end{proof}

\subsubsection{Proof of Lemma~4}

\begin{proof}
Let
\[
T=\operatorname{tr}(A)>0,
\qquad
A_\alpha=A+\alpha B.
\]
Let \(P_A\) be the orthogonal projector onto \(\operatorname{range}(A)\), and let
\[
Q_A=I-P_A.
\]
By assumption,
\[
\tau:=\operatorname{tr}(Q_A B Q_A)>0.
\]
Thus the perturbation \(B\) places positive trace mass in directions belonging
to the zero eigenspace of \(A\).

The original positive eigenvalues of \(A\) change by \(O(\alpha)\). Therefore,
their contribution to the normalized spectral entropy changes by \(O(\alpha)\).
Meanwhile, the total unnormalized spectral mass emerging from the zero
eigenspace is
\[
\alpha\tau+O(\alpha^2).
\]
After trace normalization, the total normalized mass assigned to the newly
activated eigenvalues is
\[
\varepsilon_\alpha
=
\frac{\alpha\tau+O(\alpha^2)}
{T+\alpha\operatorname{tr}(B)}
=
\frac{\tau}{T}\alpha+O(\alpha^2).
\]

Let the normalized masses assigned to the newly activated eigenvalues be
\[
q_1^{\mathrm{new}},\ldots,q_m^{\mathrm{new}},
\qquad
\sum_{\ell=1}^m q_\ell^{\mathrm{new}}
=
\varepsilon_\alpha.
\]
Write
\[
q_\ell^{\mathrm{new}}
=
\varepsilon_\alpha r_\ell,
\qquad
\sum_{\ell=1}^m r_\ell=1.
\]
Then the entropy contribution of the newly activated eigenvalues is
\[
-\sum_{\ell=1}^m
q_\ell^{\mathrm{new}}\log q_\ell^{\mathrm{new}}
=
-\sum_{\ell=1}^m
\varepsilon_\alpha r_\ell
\log(\varepsilon_\alpha r_\ell).
\]
Expanding this expression gives
\[
-\sum_{\ell=1}^m
q_\ell^{\mathrm{new}}\log q_\ell^{\mathrm{new}}
=
-\varepsilon_\alpha\log\varepsilon_\alpha
+
\varepsilon_\alpha
\left(
-\sum_{\ell=1}^m r_\ell\log r_\ell
\right).
\]
Since
\[
-\sum_{\ell=1}^m r_\ell\log r_\ell \ge 0,
\]
we obtain the lower bound
\[
-\sum_{\ell=1}^m
q_\ell^{\mathrm{new}}\log q_\ell^{\mathrm{new}}
\ge
-\varepsilon_\alpha\log\varepsilon_\alpha.
\]
Using
\[
\varepsilon_\alpha
=
\frac{\tau}{T}\alpha+O(\alpha^2),
\]
we have
\[
-\varepsilon_\alpha\log\varepsilon_\alpha
=
\frac{\tau}{T}\alpha\log(1/\alpha)+O(\alpha).
\]
Thus the newly activated zero-to-positive spectrum contributes at least a
positive \(\alpha\log(1/\alpha)\) term to the normalized spectral entropy.

The contribution from the pre-existing positive spectrum changes only by
\(O(\alpha)\). Since
\[
\alpha\log(1/\alpha)\gg \alpha
\qquad
\text{as } \alpha\downarrow 0,
\]
the positive contribution from the newly activated spectral directions dominates
the \(O(\alpha)\) change from the old positive spectrum. Hence, for all
sufficiently small \(\alpha>0\),
\[
H(A+\alpha B)>H(A).
\]
Therefore,
\[
\operatorname{er}(A+\alpha B)
=
\exp(H(A+\alpha B))
>
\exp(H(A))
=
\operatorname{er}(A).
\]
\end{proof}

\subsubsection{Proof of Lemma~5}

\begin{proof}

Throughout the proof, we condition on the fixed checkpoint and fixed
prediction-representation pairs \(\{(\hat p_i,\tilde z_i)\}_{i=1}^n\). The only
randomness is over the replacement labels \(\{\tilde y_i\}_{i\in N}\).

Since \(Q_C\) is an orthogonal projector,
\[
\operatorname{tr}(Q_C S_N^{\mathrm{glob}} Q_C)
=
\operatorname{tr}
\left(
Q_C
\sum_{i\in N}
(g_i(\tilde y_i)-\mu)(g_i(\tilde y_i)-\mu)^\top
Q_C
\right).
\]
Using linearity of trace,
\[
\operatorname{tr}(Q_C S_N^{\mathrm{glob}} Q_C)
=
\sum_{i\in N}
\operatorname{tr}
\left(
Q_C(g_i(\tilde y_i)-\mu)(g_i(\tilde y_i)-\mu)^\top Q_C
\right).
\]
For each term,
\[
\operatorname{tr}
\left(
Q_C(g_i(\tilde y_i)-\mu)(g_i(\tilde y_i)-\mu)^\top Q_C
\right)
=
\|Q_C(g_i(\tilde y_i)-\mu)\|^2.
\]
Therefore,
\[
\operatorname{tr}(Q_C S_N^{\mathrm{glob}} Q_C)
=
\sum_{i\in N}
\|Q_C(g_i(\tilde y_i)-\mu)\|^2.
\]
Taking expectation gives the claimed identity. Since each term is nonnegative,
the expectation is strictly positive whenever at least one corrupted centered
gradient has a nonzero \(Q_C\)-projection with positive probability.
\end{proof}

\subsubsection{Proof of Theorem~1}

\begin{proof}
For Case 1, apply Lemma~\ref{lem:new-directions-entropy} with
\(A=S_C^{\mathrm{glob}}\) and \(B=S_N^{\mathrm{glob}}\). Since
\[
\operatorname{tr}(Q_C S_N^{\mathrm{glob}} Q_C)>0,
\]
the lemma implies that there exists \(\alpha_0>0\) such that, for all
\(\alpha\in(0,\alpha_0)\),
\[
\operatorname{er}(S_C^{\mathrm{glob}}+\alpha S_N^{\mathrm{glob}})
>
\operatorname{er}(S_C^{\mathrm{glob}}).
\]

Now consider Case 2. Since the positive eigenvalues of
\(S_C^{\mathrm{glob}}\) are simple and \(S_N^{\mathrm{glob}}\) is supported inside
\(\operatorname{range}(S_C^{\mathrm{glob}})\), no new positive eigenvalues are
created for sufficiently small \(\alpha\). Standard first-order eigenvalue
perturbation gives
\[
\left.
\frac{d}{d\alpha}
\lambda_k(S_C^{\mathrm{glob}}+\alpha S_N^{\mathrm{glob}})
\right|_{\alpha=0}
=
u_k^\top S_N^{\mathrm{glob}} u_k.
\]
Using Lemma~\ref{lem:effective-rank-derivative},
\[
\left.
\frac{d}{d\alpha}
\operatorname{er}(S_C^{\mathrm{glob}}+\alpha S_N^{\mathrm{glob}})
\right|_{\alpha=0}
=
\sum_{k=1}^r
\beta_k(S_C^{\mathrm{glob}})
u_k^\top S_N^{\mathrm{glob}} u_k.
\]
This derivative is positive by assumption. Hence, by continuity, effective rank
increases for all sufficiently small positive \(\alpha\).
\end{proof}

\subsubsection{Proof of Theorem~2}

\begin{proof}
By Lemma~\ref{lem:first-order-loo},
\[
\Delta_i
=
a\bar g_i^\top B_S\bar g_i
+
R_i,
\]
where
\[
|R_i|
\le
C(S)\|\bar g_i\|^4
\le
C(S)\eta^2.
\]
Taking conditional expectations gives
\[
\mathbb{E}[\Delta_i\mid A]
=
a\mathbb{E}\left[\bar g_i^\top B_S\bar g_i\mid A\right]
+
\mathbb{E}[R_i\mid A],
\]
for \(A\in\{C,N\}\). Since
\[
\mathbb{E}\left[\bar g_i^\top B_S\bar g_i\mid A\right]
=
\operatorname{tr}(B_S\Sigma_A),
\]
we have
\[
\mathbb{E}[\Delta_i\mid A]
=
a\operatorname{tr}(B_S\Sigma_A)
+
\mathbb{E}[R_i\mid A].
\]
Therefore,
\[
\mathbb{E}[\Delta_i\mid N]
-
\mathbb{E}[\Delta_i\mid C]
=
a\operatorname{tr}\left(B_S(\Sigma_N-\Sigma_C)\right)
+
\mathbb{E}[R_i\mid N]
-
\mathbb{E}[R_i\mid C].
\]
The remainder difference is bounded below by
\[
\mathbb{E}[R_i\mid N]-\mathbb{E}[R_i\mid C]
\ge
-2C(S)\eta^2.
\]
Hence
\[
\mathbb{E}[\Delta_i\mid N]
-
\mathbb{E}[\Delta_i\mid C]
\ge
a\operatorname{tr}\left(B_S(\Sigma_N-\Sigma_C)\right)
-
2C(S)\eta^2.
\]
The assumed condition
\[
\operatorname{tr}\left(B_S(\Sigma_N-\Sigma_C)\right)
>
\frac{2C(S)}{a}\eta^2
\]
makes the right-hand side positive. Therefore,
\[
\mathbb{E}[\Delta_i\mid i\in N]
>
\mathbb{E}[\Delta_i\mid i\in C].
\]
\end{proof}

\subsubsection{Proof of Corollary~1}

\begin{proof}
The positive-sensitivity contribution satisfies
\[
\sum_{k\in K_+(S)}
\beta_k(S)
\left(
\mathbb{E}\left[(u_k^\top \bar g_i)^2\mid i\in N\right]
-
\mathbb{E}\left[(u_k^\top \bar g_i)^2\mid i\in C\right]
\right)
\ge
b_+\delta_+.
\]
The complementary contribution is at least \(-B_-\) by assumption. Therefore,
\[
\operatorname{tr}\left(B_S(\Sigma_N-\Sigma_C)\right)
\ge
b_+\delta_+ - B_-.
\]
If
\[
b_+\delta_+ - B_- > \frac{2C(S)}{a}\eta^2,
\]
then
\[
\operatorname{tr}\left(B_S(\Sigma_N-\Sigma_C)\right)
>
\frac{2C(S)}{a}\eta^2.
\]
The result follows from Theorem~\ref{thm:corrupted-larger-loo}.
\end{proof}

\subsubsection{Proof of Proposition~1}

\begin{proof}
Shannon entropy is Schur-concave. Therefore,
\[
p(\widetilde S_N)\prec p(\widetilde S_C)
\quad\Longrightarrow\quad
H(p(\widetilde S_N))\ge H(p(\widetilde S_C)).
\]
Exponentiating both sides gives
\[
\operatorname{er}(\widetilde S_N)
\ge
\operatorname{er}(\widetilde S_C).
\]
The claim for \(NIR\) follows immediately. Strict inequality follows from strict
Schur-concavity of entropy when the spectra are not permutations of one another.
\end{proof}

\subsubsection{Proof of Theorem~3}

\begin{proof}
Effective rank depends only on the eigenvalues of the normalized matrix
\(\rho^{(t)}\):
\[
\log \operatorname{er}(S^{(t)})=H(\rho^{(t)}).
\]
The trace-one positive semidefinite cone is compact in finite dimension, and
entropy is continuous on this cone, with the convention \(0\log 0=0\).
Therefore, entropy is uniformly continuous. Hence,
\[
\max_i\left\|\rho^{(t)}-\rho_{-i}^{(t)}\right\|_1\to 0
\quad\Longrightarrow\quad
\max_i
\left|H(\rho^{(t)})-H(\rho_{-i}^{(t)})\right|\to 0.
\]
Exponentiation is continuous on bounded intervals, so
\[
\max_i
\left|
\operatorname{er}(S^{(t)})
-
\operatorname{er}(S_{-i}^{(t)})
\right|
\to 0.
\]
Thus
\[
\max_i|\Delta_i^{(t)}|\to 0.
\]

It remains to prove the large sample bound. By Lemma~\ref{lem:centered-deletion},
\[
S^{(t)}-S_{-i}^{(t)}
=
\frac{n}{n-1}\bar g_i^{(t)}\bar g_i^{(t)\top}.
\]
Therefore,
\[
\left\|S^{(t)}-S_{-i}^{(t)}\right\|_1
=
\frac{n}{n-1}\left\|\bar g_i^{(t)}\right\|^2.
\]
Moreover,
\[
\operatorname{tr}(S^{(t)})-\operatorname{tr}(S_{-i}^{(t)})
=
\frac{n}{n-1}\left\|\bar g_i^{(t)}\right\|^2.
\]
Let
\[
\delta_i^{(t)}
=
\frac{n}{n-1}\left\|\bar g_i^{(t)}\right\|^2.
\]
Since \(S_{-i}^{(t)}\succeq 0\), we have
\[
\left\|
\frac{S^{(t)}}{\operatorname{tr}(S^{(t)})}
-
\frac{S_{-i}^{(t)}}{\operatorname{tr}(S_{-i}^{(t)})}
\right\|_1
\le
\frac{\left\|S^{(t)}-S_{-i}^{(t)}\right\|_1}
{\operatorname{tr}(S^{(t)})}
+
\left\|S_{-i}^{(t)}\right\|_1
\left|
\frac{1}{\operatorname{tr}(S^{(t)})}
-
\frac{1}{\operatorname{tr}(S_{-i}^{(t)})}
\right|.
\]
Because \(\left\|S_{-i}^{(t)}\right\|_1=\operatorname{tr}(S_{-i}^{(t)})\), this becomes
\[
\left\|
\frac{S^{(t)}}{\operatorname{tr}(S^{(t)})}
-
\frac{S_{-i}^{(t)}}{\operatorname{tr}(S_{-i}^{(t)})}
\right\|_1
\le
\frac{\delta_i^{(t)}}{\operatorname{tr}(S^{(t)})}
+
\frac{
\left|\operatorname{tr}(S^{(t)})-\operatorname{tr}(S_{-i}^{(t)})\right|
}
{\operatorname{tr}(S^{(t)})}
=
2\frac{\delta_i^{(t)}}{\operatorname{tr}(S^{(t)})}.
\]
Thus,
\[
\left\|
\rho^{(t)}-\rho_{-i}^{(t)}
\right\|_1
\leq
2\frac{n}{n-1}
\frac{\|\bar g_i^{(t)}\|^2}
{\operatorname{tr}(S^{(t)})}.
\]

Therefore, if
\[
\max_i
\frac{\|\bar g_i^{(t)}\|^2}
{\operatorname{tr}(S^{(t)})}
=
O(1/n)
\]
uniformly over the epochs under consideration, then
\[
\max_i
\left\|
\rho^{(t)}-\rho_{-i}^{(t)}
\right\|_1
=
O(1/n).
\]
Since entropy is uniformly continuous on the finite dimensional set of
trace one positive semidefinite matrices, the corresponding leave one out
effective rank contributions become small as $n$ grows.
\end{proof}

\subsection{Training Details and Reproducibility}
\label{app:training-details}

All experiments use controlled synthetic label corruption on CIFAR training subsets.
Unless otherwise stated, we use a training subset of 20,000 examples sampled without
replacement. For each run, the training subset, corruption pattern, model
initialization, minibatch order, and stochastic data augmentations are determined by
the run seed. The clean label is retained separately for every selected training
example, while the observed training label is replaced according to the label-noise
protocol below. Evaluation on the test set is always performed with the original clean
labels.

The training dataloader uses standard CIFAR data augmentation: random cropping with
padding 4 and random horizontal flipping. Fisher-rank evaluation, leave-one-out
analysis, and clean test evaluation use non-augmented inputs. In particular, the
evaluation transform only converts images to tensors, except for the CIFAR-100
ResNet18 experiments where the standard CIFAR-100 normalization is also applied.

\paragraph{Label-noise protocol.}
Let \(\mathcal{D}=\{(x_i,y_i)\}_{i=1}^n\) denote the clean training subset, where
\(y_i \in \{1,\ldots,K\}\). For a target corruption rate \(\rho\), we sample a
corruption index set
\[
\mathcal{I}_\rho \subset \{1,\ldots,n\},
\qquad
|\mathcal{I}_\rho| = \lfloor \rho n \rfloor,
\]
uniformly without replacement. We then construct the corrupted training set
\(\widetilde{\mathcal{D}}=\{(x_i,\widetilde{y}_i)\}_{i=1}^n\) by setting
\[
\widetilde{y}_i =
\begin{cases}
y_i, & i \notin \mathcal{I}_\rho, \\
\text{Uniform}\big(\{1,\ldots,K\}\setminus\{y_i\}\big), & i \in \mathcal{I}_\rho.
\end{cases}
\]
We retain the corruption indicator
\[
m_i = \mathbf{1}\{i \in \mathcal{I}_\rho\}
=
\mathbf{1}\{\widetilde{y}_i \neq y_i\}
\]
for analysis only. The model is trained using the observed labels
\(\widetilde{y}_i\), while clean-label accuracy is computed using \(y_i\).

\paragraph{Architectures.}
For CIFAR-10 ResNet18, we use a ResNet18 backbone trained from scratch. The first
convolution is modified to use a \(3\times 3\) kernel with stride 1 and padding 1,
and the initial max-pooling layer is removed. The original fully connected layer is
replaced by an identity map, followed by a separate linear classifier from the
512-dimensional representation to 10 classes.

For CIFAR-100 ResNet18, we use the same CIFAR-adapted ResNet18 architecture, except
that the classifier maps the 512-dimensional representation to 100 classes.

For the SmallCNN experiments, we use a three-block convolutional network. The first
block has two \(3\times 3\) convolutional layers with 64 channels, batch normalization,
and ReLU activations, followed by max pooling. The second block has two \(3\times 3\)
convolutional layers with 128 channels, batch normalization, and ReLU activations,
followed by max pooling. The third block has a \(3\times 3\) convolutional layer with
256 channels, batch normalization, and ReLU activation, followed by global adaptive
average pooling. The resulting 256-dimensional feature vector is projected to a
128-dimensional representation and classified using a linear layer.

For the ViT experiment, we use a compact Vision Transformer trained from scratch on
CIFAR-10. Images of size \(32\times 32\) are divided into \(4\times 4\) patches,
giving 64 patches per image. The model uses embedding dimension 384, depth 6,
6 attention heads, MLP dimension 1536, dropout 0.1, learned positional embeddings,
and a learned class token. The final classifier maps the class-token representation
to 10 classes.

\paragraph{Optimization.}
For SmallCNN, CIFAR-10 ResNet18, and CIFAR-100 ResNet18, we use stochastic gradient
descent with momentum 0.9, learning rate 0.05, and weight decay \(5\times 10^{-4}\).
The learning rate is scheduled using cosine annealing over the full training horizon.
For the CIFAR-10 ViT experiment, we use AdamW with learning rate \(3\times 10^{-4}\)
and weight decay \(5\times 10^{-2}\), again with cosine annealing over the full
training horizon. All experiments use batch size 128. Unless otherwise stated,
Fisher statistics are computed every 5 epochs.

\begin{table}[h]
\centering
\small
\caption{Experiment-specific training details. All experiments use batch size 128,
cosine annealing, and Fisher evaluation every 5 epochs. LOO denotes true leave-one-out
rank contribution analysis at the peak-rank checkpoint.}
\label{tab:training-hyperparams}
\begin{tabular}{lccccc}
\toprule
Experiment & Noise setting & Seeds & Epochs & Fisher samples & LOO samples \\
\midrule
CIFAR-10 SmallCNN
& \(\{0.5\}\)
& \(\{42,43,44,45,46\}\)
& 100
& 4096
& 2048 \\

CIFAR-10 ResNet18
& \(\{0.5\}\)
& \(\{42,43,44,45,46\}\)
& 100
& 4096
& 2048 \\

CIFAR-10 ViT
& \(\{0.5\}\)
& \(\{42,43,44,45,46\}\)
& 200
& 4096
& 2048 \\

CIFAR-100 ResNet18
& \(\{0.5\}\)
& \(\{42,43,44,45,46\}\)
& 100
& 2048
& 2048 \\

CIFAR-10 ResNet18 noise sweep
& \(\{0.0,0.2,0.3,0.4,0.5,0.6\}\)
& \(\{42,43,44\}\)
& 100
& 4096
& -- \\
CIFAR-10N ResNet18
& \texttt{worse\_label}
& \{42,43,44,45,46\}
& 100
& 4096
& 2048 \\
\bottomrule
\end{tabular}
\end{table}

\paragraph{Noise-rate sweep.}
To test whether Fisher Rank Inflation strengthens with corruption severity, we run a
separate CIFAR-10 ResNet18 noise-rate sweep with
\[
\rho \in \{0.0,0.2,0.3,0.4,0.5,0.6\}
\]
and seeds \(\{42,43,44\}\), giving \(6\times 3=18\) training runs. This experiment
does not use leave-one-out analysis. For each pair \((\rho,s)\), we train a model for
100 epochs and log one summary row containing the noise rate, seed, peak Fisher
effective rank, peak-rank epoch, peak NIR, peak-NIR epoch, final NIR, final clean
test accuracy, final training accuracy on noisy labels, final training accuracy
against clean labels, rank onset, overfitting onset, and lead time.

The main noise-sweep figure reports summary statistics only. Specifically, we plot
the mean and standard deviation over the three seeds of peak Fisher effective rank
as a function of \(\rho\), and the mean and standard deviation over the three seeds
of peak NIR as a function of \(\rho\). Since NIR requires both clean and corrupted
training examples, it is reported for \(\rho>0\). The full training curves, including
Fisher effective rank versus epoch, NIR versus epoch, test accuracy versus epoch,
and training accuracy on noisy labels versus epoch, are moved to the appendix.

\paragraph{Fisher-rank computation.}
At a given checkpoint, we compute the per-example last-layer gradient matrix. Let
\(h_i\in\mathbb{R}^m\) denote the penultimate representation of example \(i\), and let
\(p_i\in\mathbb{R}^K\) denote the softmax probability vector. For the observed training
label \(\widetilde{y}_i\), the last-layer gradient is
\[
g_i
=
\left[
\operatorname{vec}\left((p_i-e_{\widetilde{y}_i})h_i^\top\right),
\;
p_i-e_{\widetilde{y}_i}
\right]
\in \mathbb{R}^{K(m+1)},
\]
where \(e_{\widetilde{y}_i}\) is the one-hot vector for the observed label. Stacking
these gradients gives
\[
G =
\begin{bmatrix}
g_1^\top \\
\vdots \\
g_M^\top
\end{bmatrix}
\in \mathbb{R}^{M\times K(m+1)}.
\]
Here \(M\) denotes the number of examples used for Fisher-rank evaluation at the checkpoint.

For any gradient matrix \(A\in\mathbb{R}^{M_A\times K(m+1)}\), we define its
row-centered version by
\[
\mathcal{C}(A)
=
A-\mathbf{1}_{M_A}\bar g_A^\top,
\qquad
\bar g_A
=
\frac{1}{M_A}\sum_{i=1}^{M_A} a_i,
\]
where \(a_i\) denotes the \(i\)-th row of \(A\). All Fisher-rank quantities are computed
from the centered last-layer gradient scatter spectrum. Specifically, we compute the
positive eigenvalues of
\[
S(A)
=
\mathcal{C}(A)\mathcal{C}(A)^\top.
\]
Equivalently, these are the squared nonzero singular values of \(\mathcal{C}(A)\).
The normalization of \(S(A)\) by a scalar factor does not affect the effective rank,
since the eigenvalues are normalized before computing entropy.

In numerical computations, eigenvalues are clamped to be nonnegative, and
eigenvalues not exceeding \(10^{-12}\) are discarded before entropy
normalization. The same tolerance is used inside the entropy computation.

Although this absolute tolerance is not strictly scale invariant, we performed
a sensitivity analysis at the peak rank checkpoints using relative thresholds
of the form $\lambda_j > \tau \lambda_{\max}$, with
$\tau \in \{10^{-14},10^{-12},10^{-10},10^{-8}\}$. Across the evaluated
datasets and architectures, the resulting effective ranks differed from those
obtained with the absolute threshold by at most approximately $3\times10^{-6}$
in relative terms, while essentially all scatter trace was retained.

Let \(\{\lambda_j(A)\}_{j=1}^{r_A}\) denote the positive eigenvalues of \(S(A)\). We
define normalized spectral weights
\[
q_j(A)
=
\frac{\lambda_j(A)}
{\sum_{\ell=1}^{r_A}\lambda_\ell(A)}.
\]
The Fisher effective rank of \(A\) is then
\[
\operatorname{erank}(A)
=
\exp\left(
-\sum_{j=1}^{r_A} q_j(A)\log q_j(A)
\right).
\]

\paragraph{Noisy-to-Clean Inflation Ratio.}
For each checkpoint, we split the per-example last-layer gradient matrix into
clean-label and corrupted-label submatrices:
\[
G_{\mathrm{clean}}
=
\begin{bmatrix}
g_i^\top : m_i=0
\end{bmatrix},
\qquad
G_{\mathrm{noisy}}
=
\begin{bmatrix}
g_i^\top : m_i=1
\end{bmatrix}.
\]
The two submatrices are centered separately before computing their spectra. That is,
\(\operatorname{erank}(G_{\mathrm{clean}})\) is computed from
\[
\mathcal{C}(G_{\mathrm{clean}})
=
G_{\mathrm{clean}}
-
\mathbf{1}_{M_{\mathrm{clean}}}\bar g_{\mathrm{clean}}^\top,
\]
where \(\bar g_{\mathrm{clean}}\) is the mean of the clean-example gradients, and
\(\operatorname{erank}(G_{\mathrm{noisy}})\) is computed from
\[
\mathcal{C}(G_{\mathrm{noisy}})
=
G_{\mathrm{noisy}}
-
\mathbf{1}_{M_{\mathrm{noisy}}}\bar g_{\mathrm{noisy}}^\top,
\]
where \(\bar g_{\mathrm{noisy}}\) is the mean of the corrupted-example gradients.
The Noisy-to-Clean Inflation Ratio is defined as
\[
\mathrm{NIR}
=
\frac{
\operatorname{erank}(G_{\mathrm{noisy}})
}{
\operatorname{erank}(G_{\mathrm{clean}})
}.
\]
Thus, NIR compares the effective ranks of the separately subset-centered clean and
corrupted last-layer gradient scatter spectra. A value \(\mathrm{NIR}>1\) indicates that corrupted examples have a more spectrally
diffuse subset-centered last-layer gradient scatter than clean examples at that checkpoint. For the clean setting \(\rho=0\), NIR is undefined
and is therefore omitted from NIR-based summaries.

\paragraph{Rank onset, overfitting onset, and lead time.}
For each run, we define the increase in effective rank relative to initialization as
\[
\Delta r_t
=
\operatorname{erank}(G_t)-\operatorname{erank}(G_0).
\]
The rank-onset epoch is the first evaluated epoch at which \(\Delta r_t\) reaches
20\% of its maximum value over training. The overfitting-onset epoch is the first
evaluated epoch at which clean test accuracy has dropped by at least 0.03 from the
best previous clean test accuracy. The lead time is defined as
\[
\text{lead time}
=
t_{\mathrm{overfit}} - t_{\mathrm{rank}}.
\]
A positive lead time indicates that Fisher-rank inflation begins before the clean
test accuracy degradation becomes visible.

\paragraph{Leave-one-out rank contribution.}
For fixed-noise attribution experiments, we evaluate leave-one-out rank contributions
at the peak-rank checkpoint. The peak-rank checkpoint is the evaluated checkpoint with
the largest Fisher effective rank. Let \(G_{-i}\) denote the matrix obtained by
removing row \(g_i^\top\) from \(G\). For each example \(i\), we recompute the row mean
of the remaining \(M-1\) gradients, re-center the remaining matrix, and then recompute
the centered scatter spectrum. Equivalently, the leave-one-out rank contribution is
\[
\Delta_i
=
\operatorname{erank}(G)
-
\operatorname{erank}(G_{-i}),
\]
where both terms use the centering operator defined above: \(\operatorname{erank}(G)\)
is computed from \(\mathcal{C}(G)\), while \(\operatorname{erank}(G_{-i})\) is computed
from \(\mathcal{C}(G_{-i})\) using the empirical mean of the remaining rows after row
\(i\) is removed. Thus, leave-one-out attribution removes the example, re-centers the
remaining gradients, and recomputes the effective rank from the resulting centered
last-layer gradient scatter spectrum. We use these scores to rank the examples included in the leave-one-out attribution subset
and evaluate whether corrupted examples are overrepresented among the highest
rank-contributing samples. The noise-rate sweep does not use leave-one-out analysis.

\paragraph{Reproducibility.}
For every run, we set the Python, NumPy, and PyTorch random seeds to the
specified run seed. The reported five-seed experiments use seeds
\(\{42,43,44,45,46\}\), while the corruption-rate sweep uses three seeds per
corruption level. Training subsets, label-corruption masks, minibatch order,
model initialization, Fisher-evaluation subsets, and leave-one-out subsets are
generated deterministically from the run seed. Fisher-rank evaluation,
leave-one-out attribution, direct spectral diagnostics, and clean test
evaluation use non-augmented inputs. All Fisher-gradient evaluations, leave-one-out computations, and test evaluations are performed
with the model in evaluation mode, so that batch-normalization statistics are fixed during spectral
measurement.

All experiments were run locally on an Apple Silicon M4 Max machine using the
PyTorch MPS backend when available, with CPU fallback enabled for operations
not supported on MPS. In notebook-based Apple Silicon runs, dataloaders use
\texttt{num\_workers=0} to avoid backend-specific multiprocessing issues. The
implementation automatically selects Apple MPS when available, otherwise CUDA,
and otherwise CPU. The logged quantities include epoch, model type, noise
probability \(\rho\), seed, train loss, training accuracy with respect to noisy
labels, training accuracy with respect to clean labels, clean test accuracy,
Fisher effective rank, clean-example Fisher effective rank, noisy-example
Fisher effective rank, NIR, rank-onset epoch, overfitting-onset epoch, and lead
time. Peak-rank leave-one-out attribution and direct spectral diagnostics are
computed from the saved seed-specific peak-rank checkpoints and the same
deterministic leave-one-out subsets.

\subsection{Additional Training Dynamics}

The main paper focuses on ResNet18 trained on CIFAR-10. In this appendix, we
provide additional training-dynamics plots for SmallCNN, Vision Transformers,
and ResNet18 trained on CIFAR-100. Across all settings, we observe the same
qualitative phenomenon: Fisher effective rank undergoes a transient expansion
during memorization before subsequently collapsing later in training.

\subsubsection{SmallCNN Dynamics}

Figure~\ref{fig:smallcnn_dynamics} shows a representative training trajectory
for SmallCNN under \(50\%\) symmetric label corruption. Similar to ResNet18,
Fisher effective rank exhibits a pronounced inflation phase followed by a later
contraction. The noisy-to-clean inflation ratio (NIR) remains above one during
much of training, indicating that corrupted gradients are more spectrally
dispersed than clean gradients in the subset-centered Fisher-gradient scatter.
As training progresses, the network increasingly fits the provided noisy
labels, while clean test accuracy remains substantially lower than training
accuracy. These observations are consistent with the Fisher Rank Inflation
mechanism described in the main text.

\begin{figure}[H]
    \centering
    \includegraphics[width=0.75\linewidth]{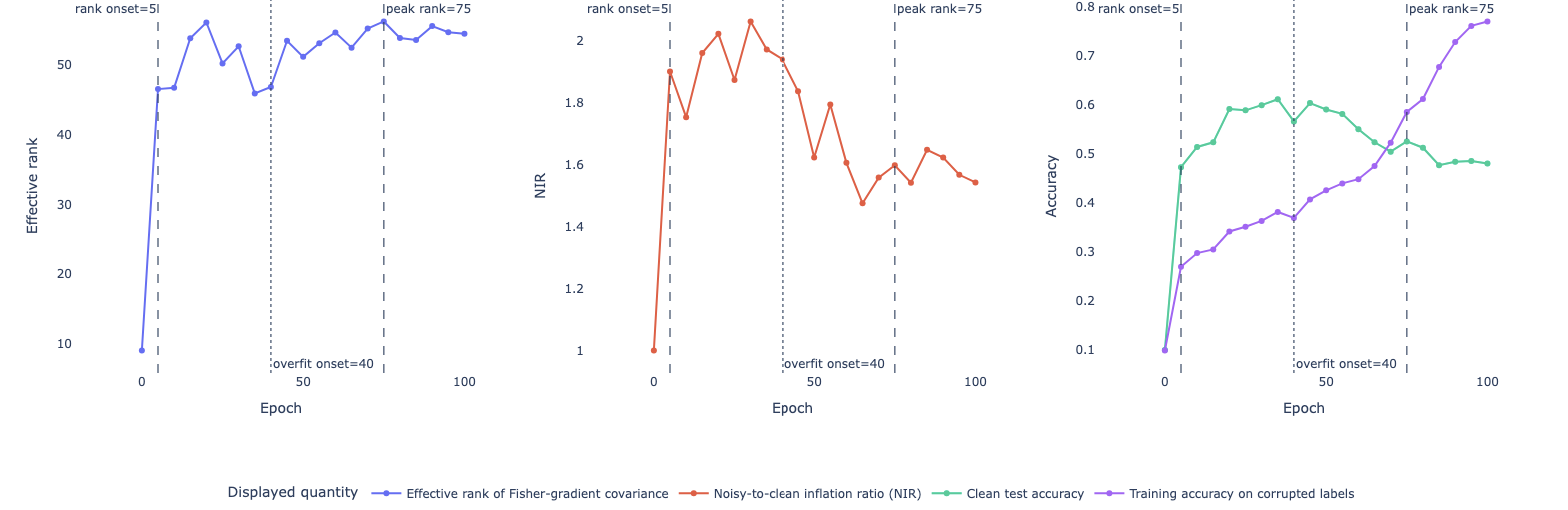}
    \caption{Representative training dynamics for SmallCNN on CIFAR-10 under \(50\%\) symmetric label corruption. Fisher effective rank exhibits an inflation--collapse trajectory similar to that observed in ResNet18. The noisy-to-clean inflation ratio (NIR) remains above one during much of training, indicating increased spectral dispersion among corrupted gradients.}
    \label{fig:smallcnn_dynamics}
\end{figure}

\subsubsection{Vision Transformer Dynamics}

Figure~\ref{fig:vit_dynamics} presents a representative training trajectory for
the Vision Transformer under \(50\%\) symmetric label corruption. Although the
detailed trajectory differs from the convolutional architectures, the same
qualitative inflation--collapse behavior is observed. Fisher effective rank
expands substantially during training before contracting at later epochs. The
noisy-to-clean inflation ratio (NIR) remains elevated during much of the
inflation phase, suggesting that corrupted gradients are more spectrally
dispersed than clean gradients even in the transformer setting. These results
indicate that Fisher Rank Inflation is not restricted to convolutional
architectures, although the multi-seed attribution results in
Table~\ref{tab:cross_arch} show that the corrupted-sample enrichment signal is
weaker for ViT than for the convolutional models.

\begin{figure}[H]
    \centering
    \includegraphics[width=0.75\linewidth]{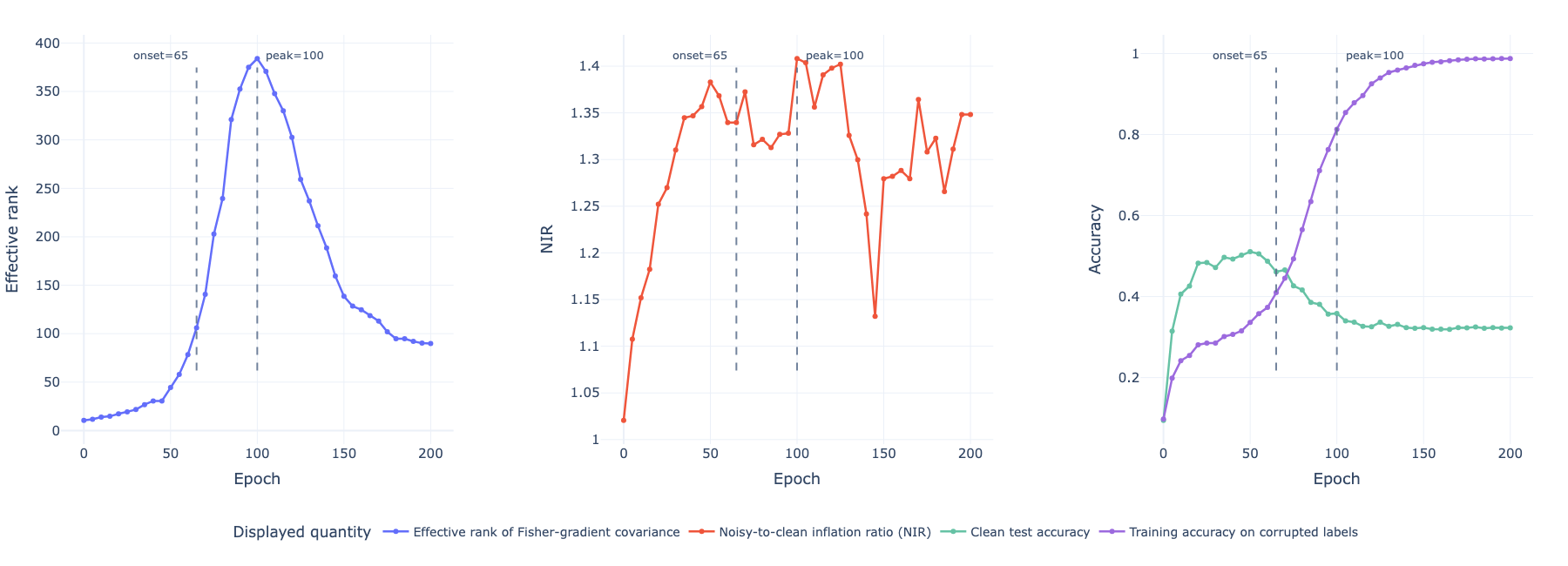}
    \caption{Representative training dynamics for a Vision Transformer on CIFAR-10 under \(50\%\) symmetric label corruption. Fisher effective rank expands during training and contracts later, showing that Fisher Rank Inflation persists in a transformer architecture. The dashed onset marker denotes the first detected rank-onset/overfitting-onset criterion in this representative run, while the peak-rank marker denotes the checkpoint used for representative attribution analysis.}
    \label{fig:vit_dynamics}
\end{figure}

\subsubsection{CIFAR-100 ResNet18 Dynamics}

Figure~\ref{fig:cifar100_dynamics} shows a representative training trajectory
for ResNet18 on CIFAR-100 with \(50\%\) symmetric label corruption. Despite the
increased dataset complexity and larger output space, Fisher effective rank
again exhibits a clear inflation--collapse trajectory. The noisy-to-clean
inflation ratio (NIR) remains above one for a substantial portion of training,
indicating that corrupted gradients are more spectrally dispersed than clean
gradients in the subset-centered Fisher-gradient scatter. These observations
suggest that Fisher Rank Inflation persists beyond CIFAR-10 and remains visible
in a higher-cardinality classification setting.

\begin{figure}[H]
    \centering
    \includegraphics[width=0.75\linewidth]{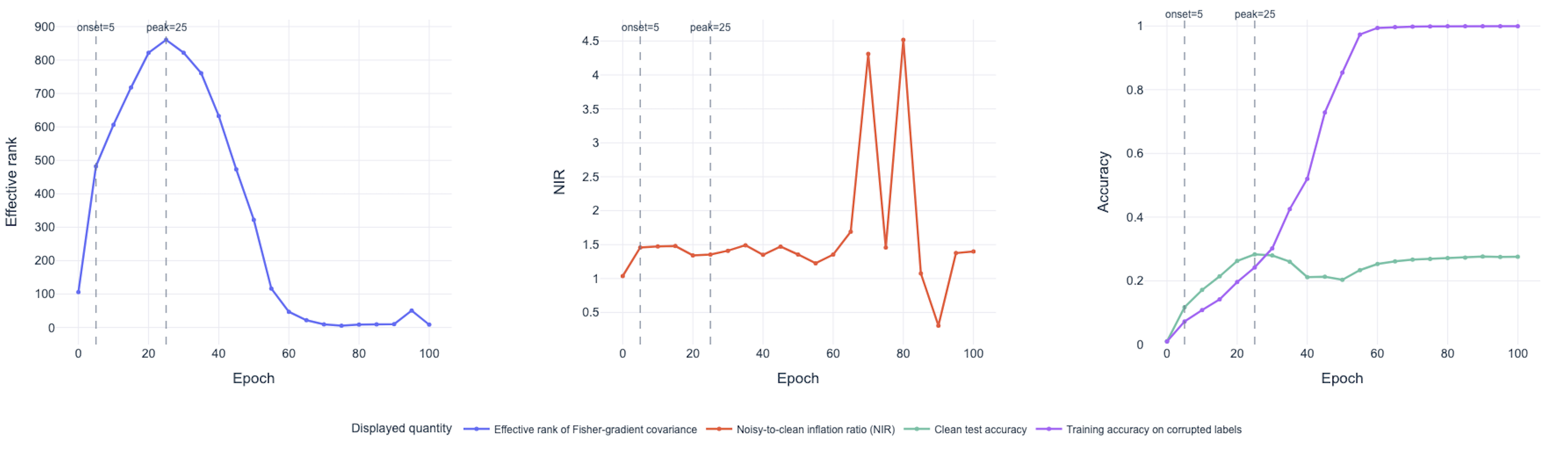}
    \caption{Representative training dynamics for ResNet18 on CIFAR-100 under \(50\%\) symmetric label corruption. Fisher effective rank exhibits a clear inflation--collapse trajectory despite the increased dataset complexity and larger output dimensionality. The noisy-to-clean inflation ratio (NIR) remains above one for a substantial portion of training, indicating greater spectral dispersion among corrupted gradients.}
    \label{fig:cifar100_dynamics}
\end{figure}

\subsection{Additional Attribution Results}

The main paper presents attribution results for ResNet18 on CIFAR-10. In this
appendix, we provide analogous analyses for SmallCNN, Vision Transformers, and
ResNet18 trained on CIFAR-100. Across all settings, we observe the same
qualitative behavior: corrupted examples exhibit systematically larger
leave-one-out rank contributions than clean examples and are strongly enriched
among the highest rank-contributing samples.

\subsubsection{SmallCNN Attribution Results}

Figure~\ref{fig:smallcnn_attr} shows representative leave-one-out rank
contributions at the peak-rank checkpoint for SmallCNN trained on CIFAR-10 with
\(50\%\) symmetric label corruption. In this run, the peak-rank checkpoint
occurs at epoch 75. Corrupted examples exhibit substantially larger positive
contributions than clean examples, with mean contributions of \(0.0051\) and
\(-0.0017\), respectively. A similar separation is observed in the median
contributions, which are \(0.0017\) for corrupted examples and \(-0.0005\) for
clean examples.

Although corrupted examples constitute only \(49.4\%\) of the evaluated subset,
they account for \(91\%\) of the top-100 rank-contributing samples,
corresponding to an enrichment factor of approximately \(1.84\times\). These
representative-run results indicate that Fisher Rank Inflation in SmallCNN is
driven disproportionately by corrupted training examples, mirroring the
behavior observed for ResNet18 in the main text. Across five seeds, the same
SmallCNN setting achieves a top-100 noisy fraction of \(96.2\%\pm2.0\%\) and an
enrichment factor of \(1.928\pm0.064\), as reported in
Table~\ref{tab:cross_arch}.

\begin{figure}[H]
    \centering
    \begin{subfigure}{0.48\textwidth}
        \centering
        \includegraphics[width=\linewidth]{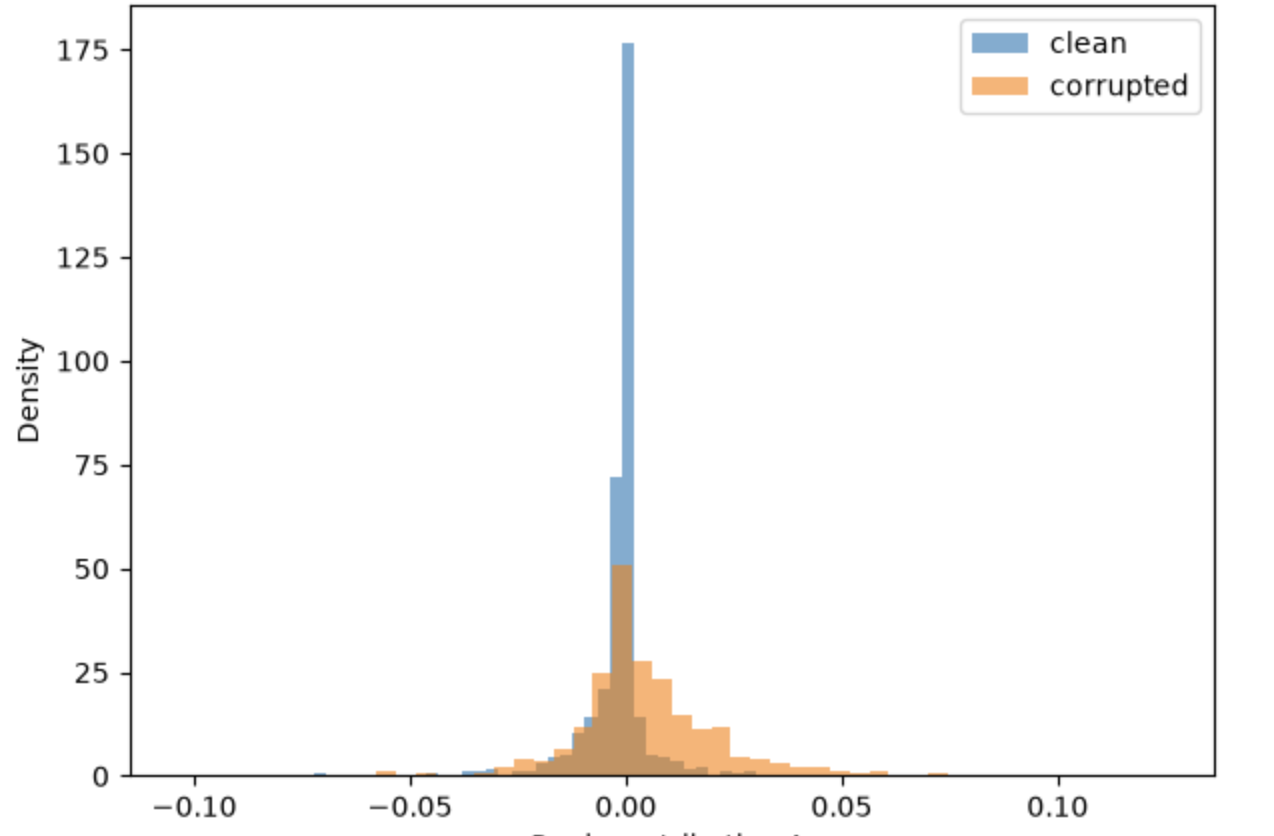}
    \end{subfigure}
    \hfill
    \begin{subfigure}{0.48\textwidth}
        \centering
        \includegraphics[width=\linewidth]{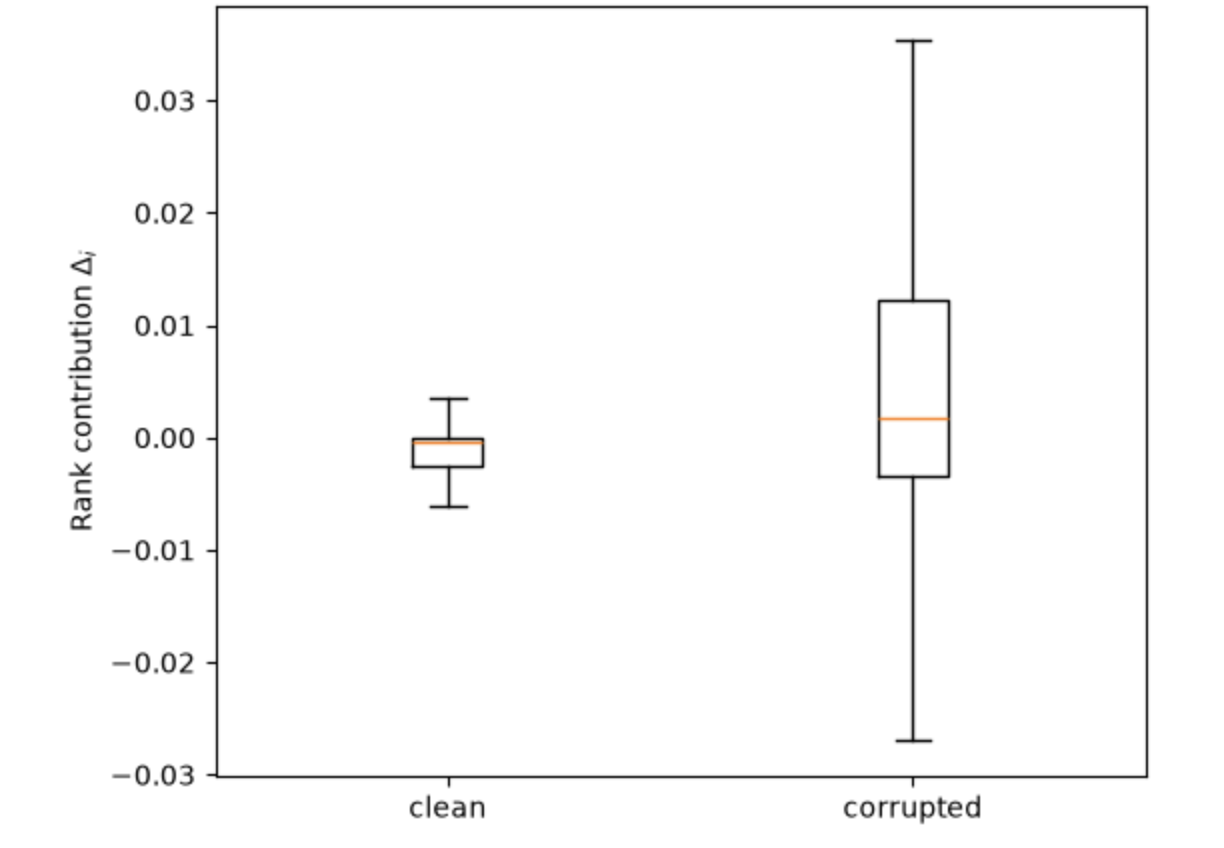}
    \end{subfigure}

    \caption{Representative leave-one-out rank contributions at the peak-rank checkpoint
for SmallCNN on CIFAR-10 with \(50\%\) symmetric label corruption. Corrupted
examples exhibit larger positive rank contributions than clean examples and are
enriched among the highest rank-contributing samples. In this representative
run, corrupted examples constitute \(49.4\%\) of the evaluated subset and
account for \(91\%\) of the top-100 rank-contributing examples.
}
    \label{fig:smallcnn_attr}
\end{figure}

\begin{table}[H]
\centering
\caption{Representative-run corruption-detection diagnostics at the peak-rank checkpoint
for SmallCNN on CIFAR-10 with \(50\%\) symmetric label noise. Multi-seed summary
statistics are reported in Table~\ref{tab:cross_arch}.}
\label{tab:smallcnn_attr}
\begin{tabular}{lccc}
\toprule
Score & AUROC & AUPRC & Top-100 noisy fraction \\
\midrule
Rank contribution $\Delta_i$ & 0.642 & 0.720 & 0.91 \\
Loss & 0.844 & 0.848 & 0.94 \\
Gradient norm & 0.807 & 0.770 & 0.80 \\
Negative confidence & 0.844 & 0.848 & 0.94 \\
Negative margin & 0.714 & 0.668 & 0.76 \\
\bottomrule
\end{tabular}
\end{table}

\subsubsection{Vision Transformer Attribution Results}

Figure~\ref{fig:vit_attr} presents representative leave-one-out rank
contributions at the peak-rank checkpoint for the Vision Transformer trained on
CIFAR-10 with \(50\%\) symmetric label corruption. Despite architectural
differences between convolutional and transformer-based models, corrupted
examples remain enriched among the highest rank-contributing samples. The
effect is weaker than in the convolutional architectures, but the same
qualitative pattern is still visible: corrupted examples tend to contribute
more strongly to Fisher effective rank than clean examples.

In this representative run, rank contribution achieves an AUROC of \(0.620\),
an AUPRC of \(0.641\), and a top-100 noisy fraction of \(0.72\). Across five
seeds, the same ViT setting achieves a top-100 noisy fraction of
\(69.2\%\pm3.7\%\) and an enrichment factor of \(1.387\pm0.087\), as reported in
Table~\ref{tab:cross_arch}. Thus, while the ViT attribution signal is
attenuated relative to the convolutional models, corrupted examples remain
overrepresented among the highest rank-contributing samples.

\begin{figure}[H]
    \centering
    \includegraphics[width=1\linewidth]{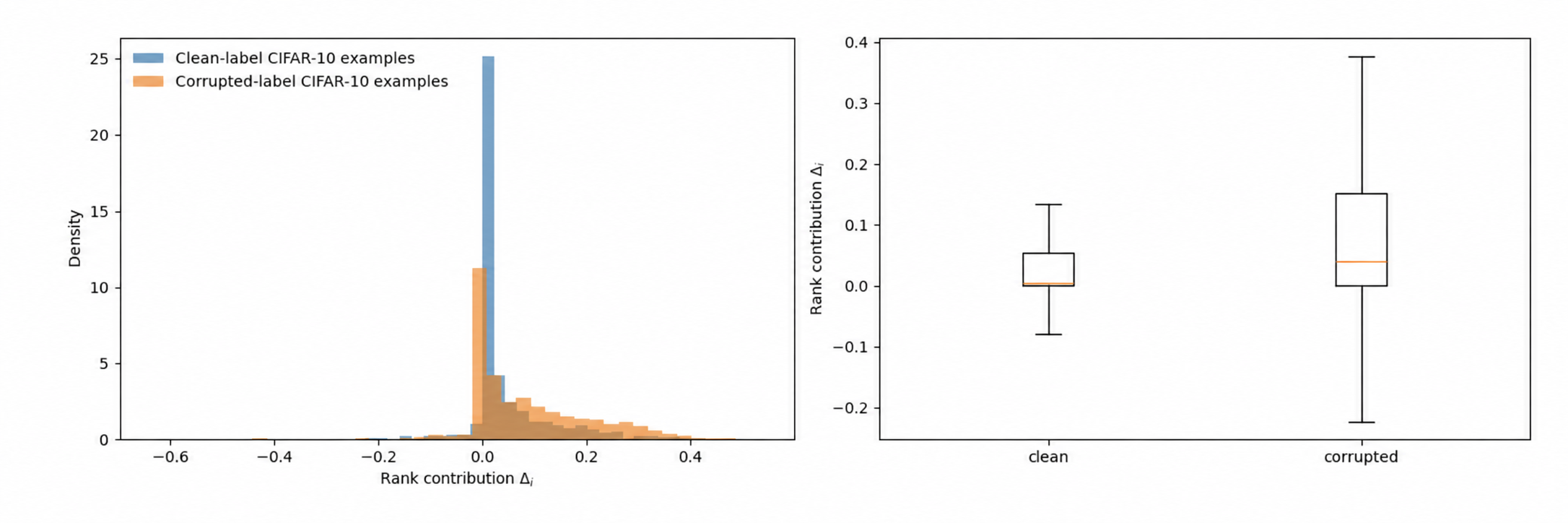}
    \caption{Representative leave-one-out rank contributions at the peak-rank checkpoint for a Vision Transformer on CIFAR-10 with \(50\%\) symmetric label corruption. The attribution signal is weaker than in the convolutional models, but corrupted examples remain enriched among the highest rank-contributing samples. }
    \label{fig:vit_attr}
\end{figure}

\begin{table}[H]
\centering
\caption{Representative-run corruption-detection diagnostics at the peak-rank checkpoint for a Vision Transformer on CIFAR-10 with \(50\%\) symmetric label noise. Multi-seed summary statistics are reported in Table~\ref{tab:cross_arch}.}
\label{tab:vit_attr}
\begin{tabular}{lccc}
\toprule
Score & AUROC & AUPRC & Top-100 noisy fraction \\
\midrule
Rank contribution $\Delta_i$ & 0.620 & 0.641 & 0.72 \\
Loss & 0.744 & 0.737 & 0.86 \\
Gradient norm & 0.742 & 0.736 & 0.88 \\
Negative confidence & 0.744 & 0.737 & 0.86 \\
Negative margin & 0.623 & 0.577 & 0.57 \\
\bottomrule
\end{tabular}
\end{table}

\subsubsection{CIFAR-100 Attribution Results}

Figure~\ref{fig:cifar100_attr} shows representative leave-one-out rank
contributions at the peak-rank checkpoint for ResNet18 trained on CIFAR-100
with \(50\%\) symmetric label corruption. Although CIFAR-100 is a more
challenging classification setting than CIFAR-10, corrupted examples again
exhibit larger rank contributions and remain enriched among the
highest-contributing samples.

In this representative run, corrupted examples account for \(82\%\) of the
top-100 rank-contributing samples despite a background corruption rate of
approximately \(50\%\). Across five seeds, the same CIFAR-100 ResNet18 setting
achieves a top-100 noisy fraction of \(92.2\%\pm1.6\%\) and an enrichment factor
of \(1.848\pm0.052\), as reported in Table~\ref{tab:cross_arch}. These results
indicate that the association between Fisher Rank Inflation and corrupted
examples extends beyond CIFAR-10 and persists in a higher-cardinality
classification setting.

\begin{figure}[H]
    \centering
    \includegraphics[width=1\linewidth]{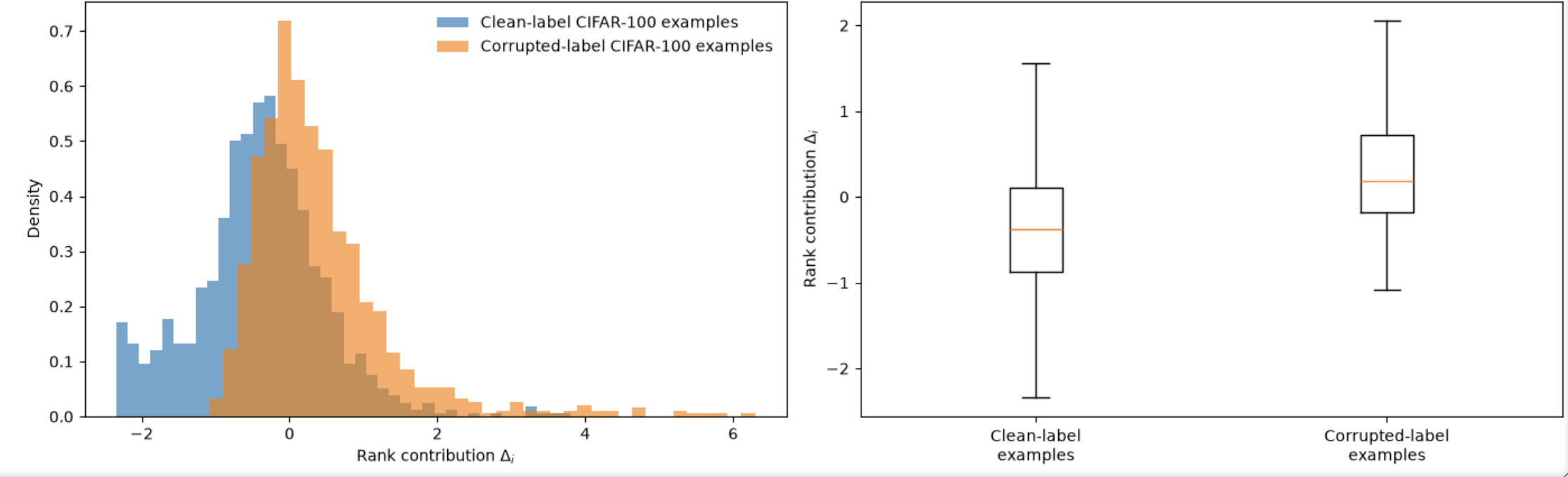}
    \caption{Representative leave-one-out rank contributions at the peak-rank checkpoint for ResNet18 on CIFAR-100 with \(50\%\) symmetric label corruption. Corrupted examples exhibit larger rank contributions and are enriched among the highest rank-contributing samples.}
    \label{fig:cifar100_attr}
\end{figure}

\begin{table}[H]
\centering
\caption{Representative-run corruption-detection diagnostics at the peak-rank checkpoint for ResNet18 on CIFAR-100 with \(50\%\) symmetric label noise. Multi-seed summary statistics are reported in Table~\ref{tab:cross_arch}.}
\label{tab:cifar100_attr}
\begin{tabular}{lccc}
\toprule
Score & AUROC & AUPRC & Top-100 noisy fraction \\
\midrule
Rank contribution $\Delta_i$ & 0.746 & 0.707 & 0.82 \\
Loss & 0.871 & 0.850 & 0.95 \\
Gradient norm & 0.720 & 0.683 & 0.77 \\
Negative confidence & 0.584 & 0.545 & 0.53 \\
\bottomrule
\end{tabular}
\end{table}

\subsection{Fisher Rank Inflation under Human Annotation Noise}

The experiments in the main paper focus on synthetic symmetric label corruption. To evaluate whether Fisher Rank Inflation extends beyond artificially generated noise, we additionally study CIFAR-10N \citep{wei2022learningnoisylabelsrevisited}, which contains naturally occurring human annotation errors collected from multiple annotators. Unlike synthetic corruption, these label errors arise from human mistakes and ambiguity rather than random label replacement.

We evaluate ResNet18 on the CIFAR-10N \texttt{worse\_label} split using the same training protocol, Fisher-rank computation, and leave-one-out attribution procedure described in Section~4.1. Models are trained using the observed human labels while clean labels are retained for evaluation and attribution analysis. All results are reported over five random seeds.

\subsubsection{Training Dynamics}

Figure~\ref{fig:cifar10n_dynamics} shows representative Fisher-rank dynamics under human annotation noise. Similar to the synthetic-corruption setting, Fisher effective rank exhibits a clear inflation--collapse trajectory during training. Effective rank initially expands as training progresses, reaches a pronounced maximum, and subsequently contracts despite continued optimization. The noisy-to-clean inflation ratio (NIR) remains elevated throughout much of training, indicating greater spectral dispersion among human-noisy examples than among clean examples.

The temporal relationship between Fisher-rank inflation and observable overfitting is also preserved. Across five seeds, Fisher-rank onset precedes the chosen overfitting-onset criterion by $29.0 \pm 13.1$ epochs on average. Thus, the early-warning behavior observed under synthetic corruption continues to appear under naturally occurring annotation errors.

These results indicate that the Fisher-spectrum dynamics identified in the main paper are not restricted to synthetic label corruption. Instead, naturally occurring annotation errors produce qualitatively similar inflation--collapse trajectories and lead-time behavior.

\begin{figure}[H]
\centering
\includegraphics[width=0.9\linewidth]{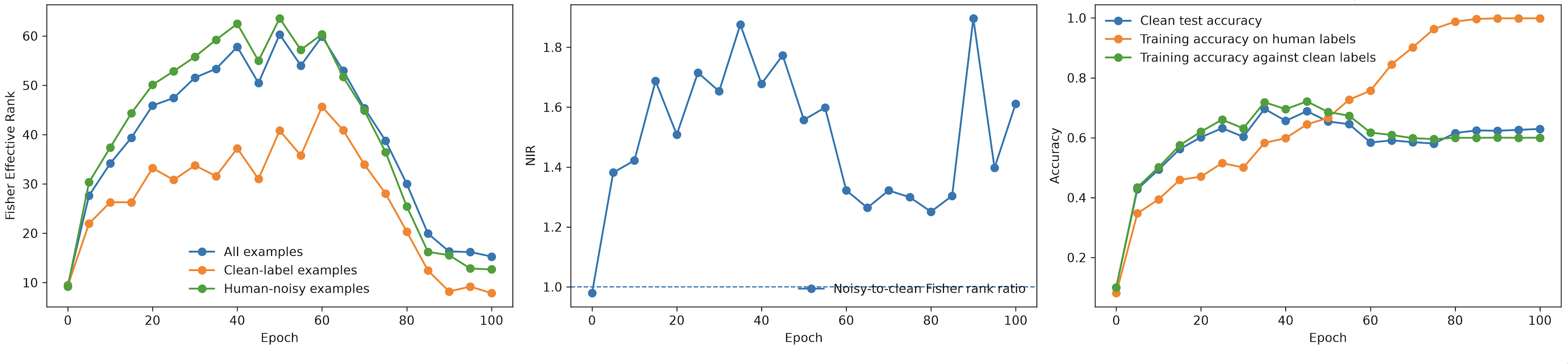}
\caption{
Representative training dynamics for ResNet18 on CIFAR-10N using the \texttt{worse\_label} split. Fisher effective rank exhibits a pronounced inflation--collapse trajectory similar to that observed under synthetic label corruption. The noisy-to-clean inflation ratio (NIR) remains above one for much of training, indicating greater spectral dispersion among human-noisy examples than among clean examples.
}
\label{fig:cifar10n_dynamics}
\end{figure}

\subsubsection{Leave-One-Out Attribution}

We next evaluate leave-one-out Fisher-rank contributions at the seed-specific peak-rank checkpoints. Figure~\ref{fig:cifar10n_attribution} shows representative attribution distributions for clean and human-noisy examples.

Human-noisy examples exhibit substantially larger positive rank contributions than clean examples. The distribution of noisy-example contributions is shifted toward larger positive values and exhibits greater variance, while clean examples remain concentrated near zero contribution. This qualitative pattern closely matches the synthetic-corruption experiments reported in the main paper.

The enrichment of human-noisy examples among the highest Fisher-rank contributors is particularly strong. Across five seeds, human-noisy examples account for $94.4\%\pm1.9\%$ of the top-100 rank-contributing samples despite constituting only $39.7\%\pm1.0\%$ of the evaluated subset. This corresponds to an enrichment factor of $2.377\pm0.093$ relative to the background human-noise rate.

These findings indicate that the examples most responsible for Fisher-rank expansion remain strongly concentrated among mislabeled examples even when the noise arises from naturally occurring annotation errors rather than synthetic label flips.

\begin{figure}[H]
\centering
\includegraphics[width=1\linewidth]{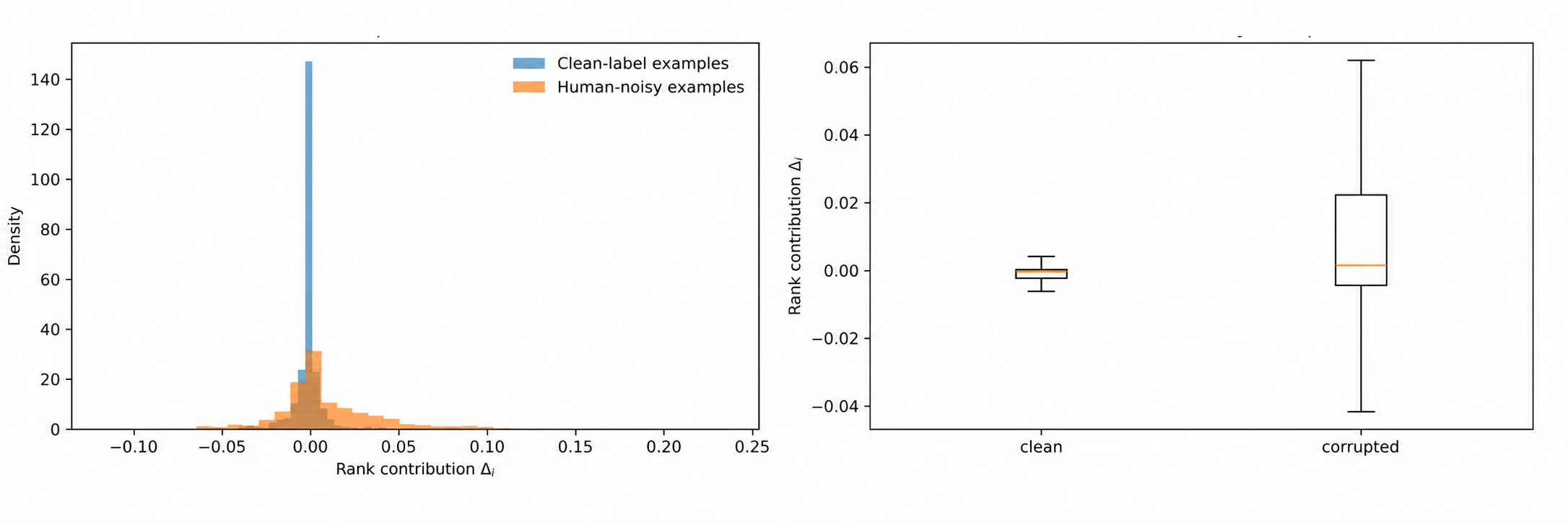}
\caption{
Representative leave-one-out rank contributions at the peak-rank checkpoint for ResNet18 on CIFAR-10N. Human-noisy examples exhibit substantially larger positive rank contributions than clean examples and are strongly enriched among the highest Fisher-rank contributors.
}
\label{fig:cifar10n_attribution}
\end{figure}

\subsection{Quantitative Results}

Table~\ref{tab:cifar10n_results} summarizes the five-seed CIFAR-10N results.

Although rank contribution is not intended as a generic noisy-label detector, it remains informative under human annotation noise. The rank-contribution score achieves an AUROC of $0.608\pm0.010$ and an AUPRC of $0.645\pm0.014$. These values are lower than those achieved by loss-based noisy-label detection baselines, but they are consistent with the interpretation of Fisher-rank attribution as a spectral diagnostic rather than a direct label-quality estimator.

More importantly, the attribution signal remains highly concentrated among the examples responsible for Fisher-spectrum expansion. The strong enrichment of human-noisy examples among the highest rank contributors indicates that Fisher Rank Inflation continues to identify the subset of mislabeled examples most strongly associated with memorization-related spectral reorganization.

\begin{table}[H]
\centering
\caption{
Five-seed Fisher Rank Inflation results on CIFAR-10N using the \texttt{worse\_label} split. Results are reported as mean $\pm$ standard deviation over five random seeds.
}
\label{tab:cifar10n_results}
\begin{tabular}{lc}
\toprule
Metric & Value \\
\midrule
Human-noise rate & $0.397 \pm 0.010$ \\
Peak Fisher effective rank & $59.39 \pm 2.50$ \\
Peak-rank epoch & $50.0 \pm 6.1$ \\
Lead time & $29.0 \pm 13.1$ \\
Top-100 noisy fraction & $0.944 \pm 0.019$ \\
Top-100 enrichment & $2.377 \pm 0.093$ \\
Rank-contribution AUROC & $0.608 \pm 0.010$ \\
Rank-contribution AUPRC & $0.645 \pm 0.014$ \\
\bottomrule
\end{tabular}
\end{table}

\subsection{Additional Noise-Sweep Results}
\label{app:noise-sweep}

The main paper summarizes the dependence of Fisher Rank Inflation on corruption
severity in Figure~\ref{fig:noise-sweep}. In this appendix, we provide
epoch-wise mean \(\pm\) standard deviation trajectories for the same
corruption-severity sweep. These plots show the full temporal evolution of
Fisher effective rank, the noisy-to-clean inflation ratio (NIR), clean test
accuracy, and training accuracy on the provided labels across corruption rates.

All curves in this section are computed over three random seeds for ResNet18 on
CIFAR-10 under symmetric label corruption rates
\[
\rho \in \{0.0,0.2,0.3,0.4,0.5,0.6\}.
\]
The shaded regions denote one standard deviation across seeds.

\subsubsection{Fisher-Spectrum Dynamics Across Noise Levels}
\label{app:noise-spectrum-dynamics}

Figure~\ref{fig:noise_sweep_er_epoch} shows the epoch-wise Fisher effective-rank
trajectories across corruption levels. Increasing label corruption produces a
larger effective-rank expansion during training, with higher corruption rates
reaching larger peak effective rank before the subsequent collapse phase. This
trajectory-level behavior supports the summary result in the main text: peak
Fisher effective rank increases with corruption severity.

\begin{figure}[H]
    \centering
    \includegraphics[width=0.75\linewidth, trim={0cm 0cm 0cm 0.7cm}, clip]{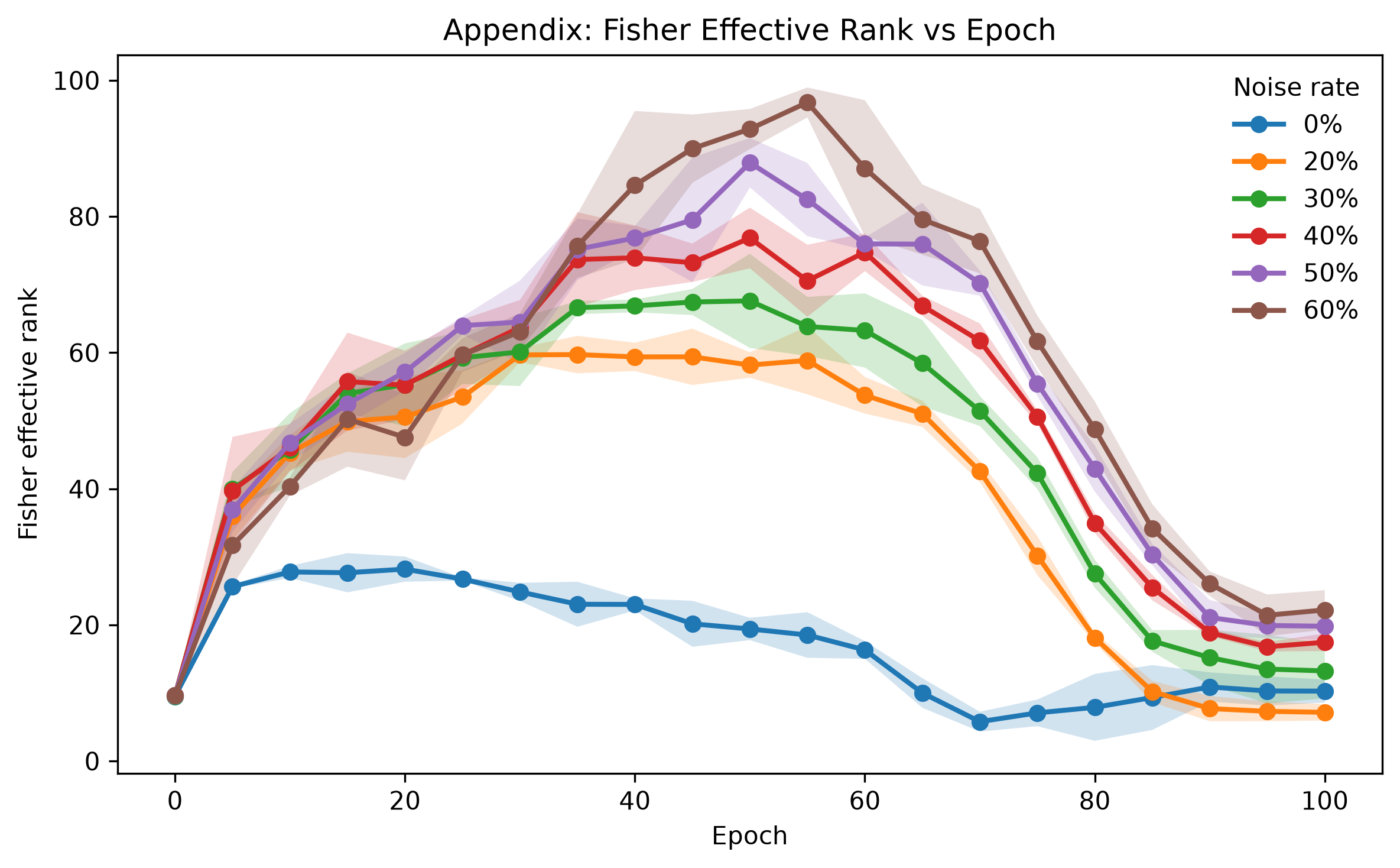}
    \caption{
    Epoch-wise Fisher effective-rank trajectories for ResNet18 on CIFAR-10
    under different symmetric label-corruption rates. Curves show mean
    \(\pm\) standard deviation across three random seeds. Higher corruption
    rates produce stronger Fisher-rank expansion during training, followed by a
    collapse phase after corrupted labels are increasingly fit.
    }
    \label{fig:noise_sweep_er_epoch}
\end{figure}

Figure~\ref{fig:noise_sweep_nir_epoch} shows the corresponding epoch-wise
noisy-to-clean inflation ratio (NIR). Across noisy settings, NIR remains above
one for much of training, indicating that corrupted-gradient subsets are more
spectrally dispersed than clean-gradient subsets in the subset-centered
Fisher-gradient scatter. Unlike peak effective rank, NIR is not strictly
monotonic in the corruption rate, but it remains elevated across noisy
conditions. NIR is undefined for the clean setting because there is no corrupted
subset.

\begin{figure}[H]
    \centering
    \includegraphics[width=0.75\linewidth, trim={0cm 0cm 0cm 0.7cm}, clip]{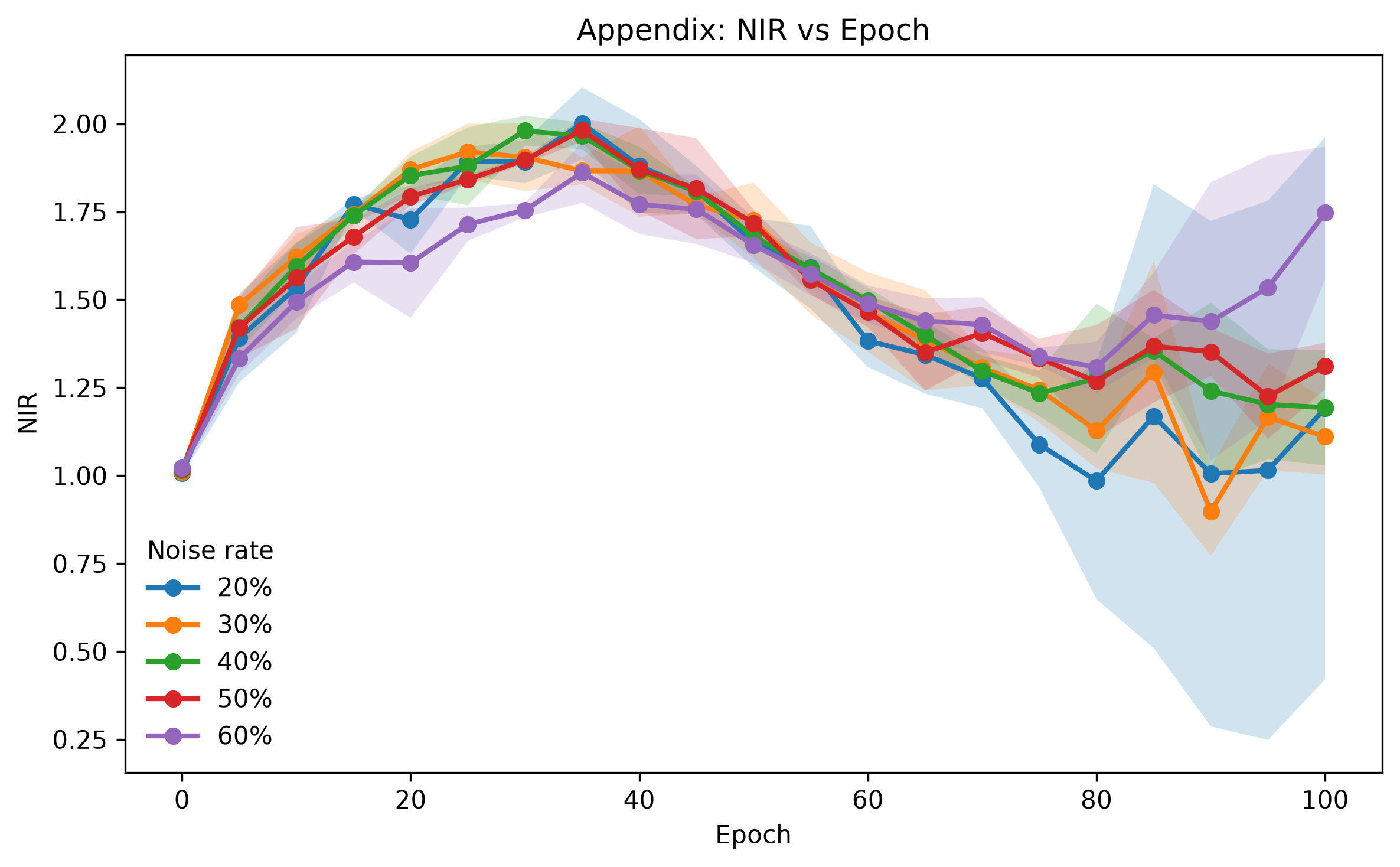}
    \caption{
    Epoch-wise noisy-to-clean inflation ratio (NIR) trajectories for ResNet18
    on CIFAR-10 under different symmetric label-corruption rates. Curves show
    mean \(\pm\) standard deviation across three random seeds. NIR remains above
    one for much of training across noisy settings, indicating greater
    subset-centered spectral dispersion among corrupted gradients than among
    clean gradients.
    }
    \label{fig:noise_sweep_nir_epoch}
\end{figure}

\subsubsection{Generalization and Memorization Across Noise Levels}
\label{app:noise-generalization-dynamics}

Figure~\ref{fig:noise_test_acc} shows the clean test-accuracy trajectories
obtained during the label-noise sweep. As the corruption rate increases, clean
test accuracy decreases and generalization degradation becomes more pronounced.
While all models initially improve during optimization, higher corruption levels
ultimately lead to lower clean test accuracy despite continued improvement on
the provided training labels.

\begin{figure}[H]
    \centering
    \includegraphics[width=0.75\linewidth, trim={0cm 0cm 0cm 0.7cm}, clip]{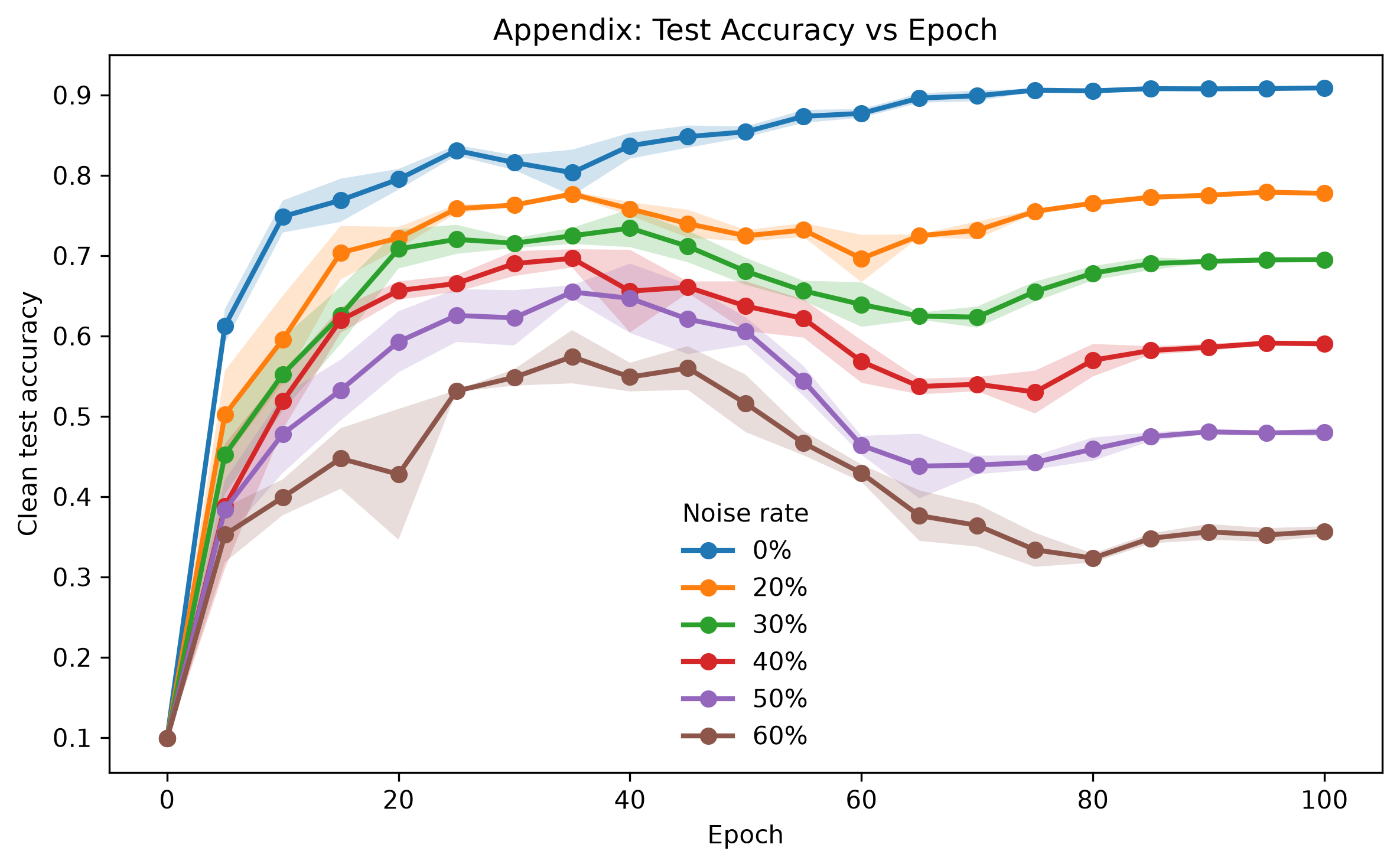}
    \caption{
    Clean test-accuracy trajectories for ResNet18 trained on CIFAR-10 under
    varying levels of symmetric label corruption. Curves show mean \(\pm\)
    standard deviation across three random seeds. Higher corruption rates lead
    to progressively lower clean test accuracy and stronger generalization
    degradation.
    }
    \label{fig:noise_test_acc}
\end{figure}

Figure~\ref{fig:noise_train_acc} reports training accuracy measured with
respect to the provided labels. Although larger corruption rates slow
optimization during the early stages of training, all noisy settings eventually
approach near-perfect training accuracy on the provided labels. This behavior
indicates that the network ultimately fits both clean and corrupted labels.
Combined with the clean test-accuracy trajectories, these results illustrate
the classical memorization phenomenon in which training performance continues
to improve even as generalization deteriorates.

\begin{figure}[H]
    \centering
    \includegraphics[width=0.82\linewidth, trim={0cm 0cm 0cm 0.7cm}, clip]{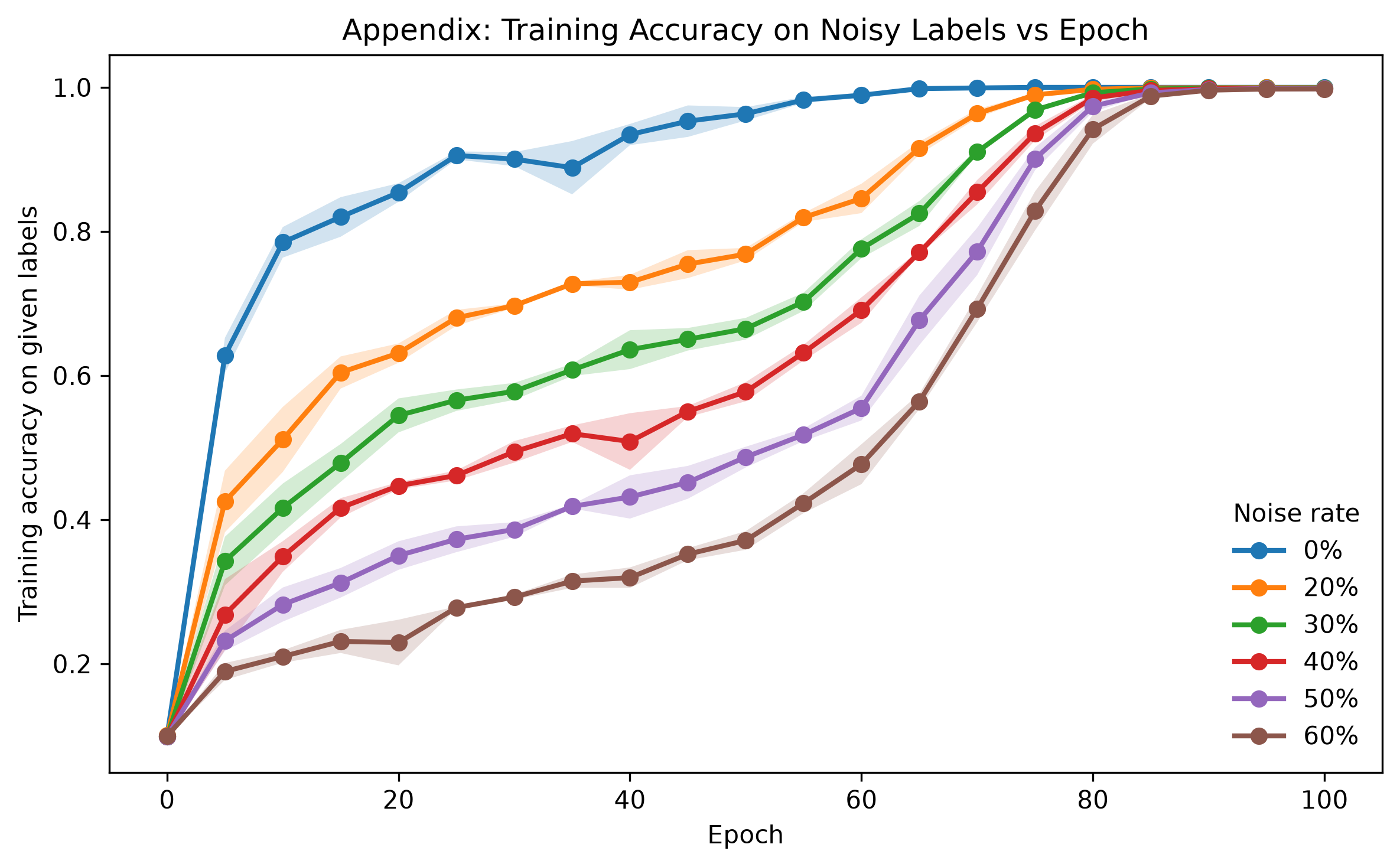}
    \caption{
    Training accuracy measured with respect to the provided labels for ResNet18
    trained on CIFAR-10 under varying levels of symmetric label corruption.
    Curves show mean \(\pm\) standard deviation across three random seeds.
    Despite substantial corruption, training accuracy approaches unity across
    noisy settings, demonstrating eventual memorization of the provided labels.
    }
    \label{fig:noise_train_acc}
\end{figure}

Together, these epoch-wise results provide additional support for the main
corruption-severity finding. Increasing label corruption produces stronger
Fisher-spectrum expansion, corrupted gradients remain more spectrally dispersed
than clean gradients for much of training, and heavier corruption leads to
worse clean generalization even as training accuracy on the provided labels
continues to improve.

\subsubsection{Lead Time Analysis}
\label{app:lead-time}

We also examine the temporal relationship between Fisher-rank onset and
observable overfitting. For each run, rank onset is defined as the first
evaluated epoch at which Fisher effective rank reaches \(20\%\) of its maximum
increase over training. Overfitting onset is defined as the first evaluated
epoch at which clean test accuracy drops by at least \(0.03\) from the best
previous clean test accuracy. The lead time is then defined as
\[
\mathrm{lead\ time}
=
t_{\mathrm{overfit}} - t_{\mathrm{rank}}.
\]

Across the noisy corruption settings, the mean lead time is positive, indicating
that Fisher-rank onset generally precedes the chosen overfitting-onset
criterion. In the three-seed corruption sweep, the mean lead times are
\(43.3\), \(35.0\), \(36.7\), \(38.3\), and \(28.3\) epochs for corruption
rates \(20\%\), \(30\%\), \(40\%\), \(50\%\), and \(60\%\), respectively. The
corresponding standard deviations are \(2.9\), \(8.7\), \(2.9\), \(2.9\), and
\(11.5\) epochs.

 These results show that the onset of Fisher Rank Inflation can occur before observable test accuracy degradation under label noise. The magnitude of the lead time varies across seeds and corruption levels. We therefore interpret this analysis as a retrospective comparison of temporal ordering.

\subsection{Additional Direct-Diagnostic Details}
\label{app:direct-diagnostic-details}

The main paper reports compact checkpoint-level diagnostics in
Table~\ref{tab:direct-spectral-diagnostics}. For completeness, the diagnostics
were computed at the same seed-specific peak-rank checkpoints used for the
leave-one-out attribution analysis. For each seed, we recomputed the
per-example final-layer gradient matrix at the saved peak-rank checkpoint,
formed the globally centered gradient matrix, and evaluated the corresponding
spectral quantities from Section~3. The new-direction fraction was computed as
\[
\frac{\operatorname{tr}(Q_C S_N^{\mathrm{glob}}Q_C)}
{\operatorname{tr}(S_N^{\mathrm{glob}})},
\]
where \(Q_C\) is the orthogonal projector onto the complement of the globally
centered clean-gradient span. The first-order score
\[
A_i(S)=\bar g_i^\top B_S\bar g_i
\]
was compared against the exact leave-one-out contribution
\[
\Delta_i=\operatorname{er}(S)-\operatorname{er}(S_{-i}).
\]
All reported values are mean and standard deviation over the same five seeds
used in the peak-rank attribution experiments.

\subsection{Onset-Threshold Sensitivity}
\label{app:onset-threshold-sensitivity}

\begin{table}[H]
\centering
\caption{Sensitivity of lead-time estimates to the choice of rank-onset and
overfitting-onset thresholds. Rank onset is evaluated using threshold fractions
\(\{0.10,0.20,0.30\}\) of the peak effective-rank increase, and overfitting
onset is evaluated using clean-test-accuracy drops
\(\{0.02,0.03,0.05\}\) from the best previous value, giving nine threshold
pairs per setting. We report the mean lead time averaged over threshold pairs,
the range of mean lead times across the grid, and the fraction of threshold
pairs for which the mean lead time is positive.}
\label{tab:onset-threshold-sensitivity}
\begin{tabular}{llccc}
\toprule
Dataset & Model
& Mean lead
& Lead range
& Positive fraction \\
\midrule
CIFAR-10 & SmallCNN
& \(22.0\)
& \([21.0,23.0]\)
& \(1.00\) \\

CIFAR-10 & ResNet18
& \(35.7\)
& \([32.0,40.0]\)
& \(1.00\) \\

CIFAR-10 & ViT
& \(6.0\)
& \([-7.0,22.0]\)
& \(0.60\) \\

CIFAR-100 & ResNet18
& \(30.7\)
& \([27.0,35.0]\)
& \(1.00\) \\
\bottomrule
\end{tabular}
\end{table}

Table~\ref{tab:onset-threshold-sensitivity} evaluates the robustness of the
lead-time analysis to the choice of rank-onset and overfitting-onset
thresholds. For the convolutional CIFAR-10 models and CIFAR-100 ResNet18, the
mean lead time remains positive for every threshold pair considered. This
indicates that the conclusion that Fisher-rank inflation precedes observable
test degradation is not an artifact of the particular \(20\%\) rank-onset
threshold or \(0.03\) test-drop threshold used in the main analysis.

The Vision Transformer exhibits a weaker and more threshold-dependent lead-time
signal: the average lead time remains positive over the grid, but some
threshold choices produce non-positive lead time. This is consistent with the
weaker ViT attribution and direct-diagnostic results reported in the main text.

\subsection{Centered versus uncentered Fisher-gradient spectra}

Our main experiments use the centered scatter of per-example last-layer
gradients, since this isolates example-to-example gradient fluctuations.
To verify that Fisher Rank Inflation is not an artifact of centering, we
also repeated the analysis using the uncentered empirical Fisher. The two
quantities produced nearly identical results on CIFAR-10 ResNet18 under
50\% symmetric label corruption. The centered scatter achieved a peak
effective rank of $85.99 \pm 3.75$, peak NIR of $2.00 \pm 0.06$, and
top-100 noisy fraction of $0.950 \pm 0.014$, while the uncentered empirical
Fisher achieved a peak effective rank of $85.34 \pm 3.86$, peak NIR of
$1.98 \pm 0.05$, and top-100 noisy fraction of $0.952 \pm 0.015$.
Similarly, the rank-contribution AUROC/AUPRC were nearly unchanged:
$0.681/0.762$ for the centered scatter and $0.681/0.762$ for the uncentered
Fisher. These results indicate that the observed inflation and corrupted-example
enrichment are not caused by the centering operation.

\begin{table}[H]
\centering
\caption{
Centered versus uncentered Fisher-gradient ablation for CIFAR-10 ResNet18
under 50\% symmetric label corruption. Results are mean $\pm$ standard
deviation over five seeds.
}
\label{tab:centered-uncentered-ablation}
\begin{tabular}{lcccccc}
\toprule
Matrix & Peak ER & Peak NIR & Top-100 Noisy & Enrichment & AUROC & AUPRC \\
\midrule
Centered scatter
& $85.99 \pm 3.75$
& $2.00 \pm 0.06$
& $0.950 \pm 0.014$
& $1.904 \pm 0.037$
& $0.681 \pm 0.024$
& $0.762 \pm 0.023$ \\
Uncentered Fisher
& $85.34 \pm 3.86$
& $1.98 \pm 0.05$
& $0.952 \pm 0.015$
& $1.908 \pm 0.039$
& $0.681 \pm 0.023$
& $0.762 \pm 0.022$ \\
\bottomrule
\end{tabular}
\end{table}

\subsection{Additional Results for the Clean-Difficulty Control}
\label{app:clean_difficulty}

This subsection provides additional evidence supporting the clean-difficulty control
experiment presented in Section~\ref{sec:clean_difficulty}.

For each random seed, we train ResNet18 on CIFAR-10 with 50\% symmetric label
corruption and compute leave-one-out Fisher-rank contributions at the seed-specific
peak Fisher-rank checkpoint. Samples are partitioned into three groups:

\begin{itemize}
    \item \textbf{Normal clean}: correctly labeled examples with relatively low
    training loss.
    \item \textbf{High-loss clean}: correctly labeled examples having the highest
    training losses among all clean samples. These serve as a proxy for rare,
    ambiguous, or intrinsically difficult yet correctly labeled examples.
    \item \textbf{Memorized corrupted}: examples whose observed labels are incorrect
    but whose training losses under those incorrect labels are already very small,
    indicating that the network has memorized the corrupted supervision.
\end{itemize}

Figure~\ref{fig:control_histogram} shows the corresponding contribution
distributions. The memorized corrupted group consistently occupies the positive
tail, while high-loss clean examples do not exhibit comparable positive
contributions.

Table~\ref{tab:control_topk} summarizes the composition of the largest
leave-one-out contributors. Averaged over three seeds, memorized corrupted
examples constitute approximately $96.7\%$, $96.0\%$, and $93.2\%$ of the
top-50, top-100, and top-200 contributors, respectively, while high-loss clean
examples remain nearly absent.

Together, these results demonstrate that Fisher Rank Inflation is not explained
by optimization difficulty alone. Instead, corrupted examples dominate the
Fisher-rank expansion once they have been memorized, even though their training
losses under the observed labels are already small.

\begin{table}[H]
\centering
\caption{Composition of the highest leave-one-out Fisher-rank contributors,
averaged across three random seeds. Entries report percentages.}
\label{tab:control_topk}
\begin{tabular}{lccc}
\toprule
Top-$k$ & Memorized corrupted & High-loss clean & Normal clean\\
\midrule
50  & $96.7\pm4.2$ & $1.3\pm1.2$ & $2.0\pm3.5$\\
100 & $96.0\pm2.6$ & $0.7\pm0.6$ & $3.3\pm2.1$\\
200 & $93.2\pm2.4$ & $1.5\pm0.5$ & $5.3\pm2.3$\\
\bottomrule
\end{tabular}
\end{table}

\begin{figure}[H]
\centering
\begin{subfigure}{0.32\linewidth}
    \centering
    \includegraphics[width=\linewidth]{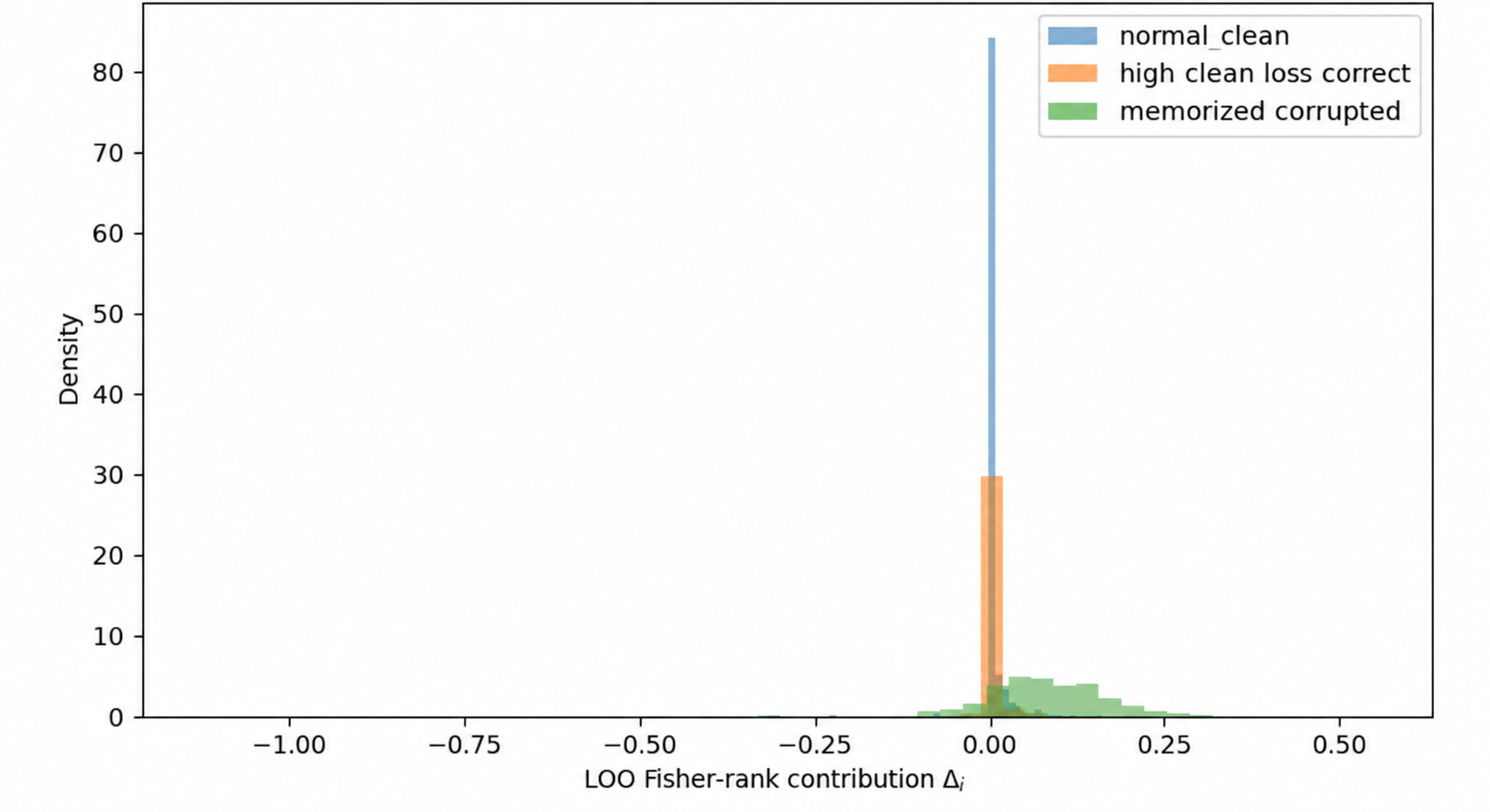}
    \caption{Seed 42}
\end{subfigure}
\hfill
\begin{subfigure}{0.32\linewidth}
    \centering
    \includegraphics[width=\linewidth]{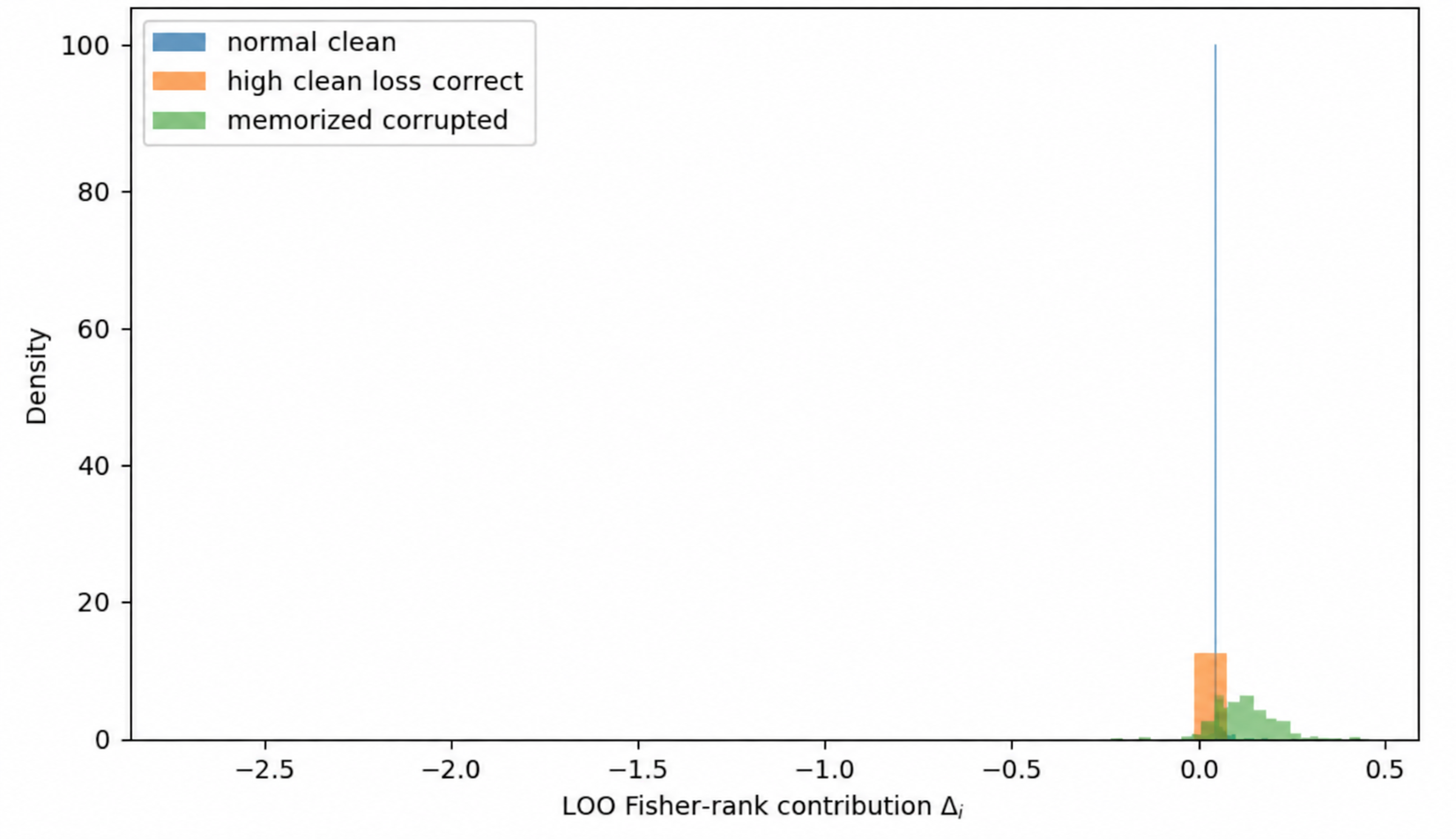}
    \caption{Seed 43}
\end{subfigure}
\hfill
\begin{subfigure}{0.32\linewidth}
    \centering
    \includegraphics[width=\linewidth]{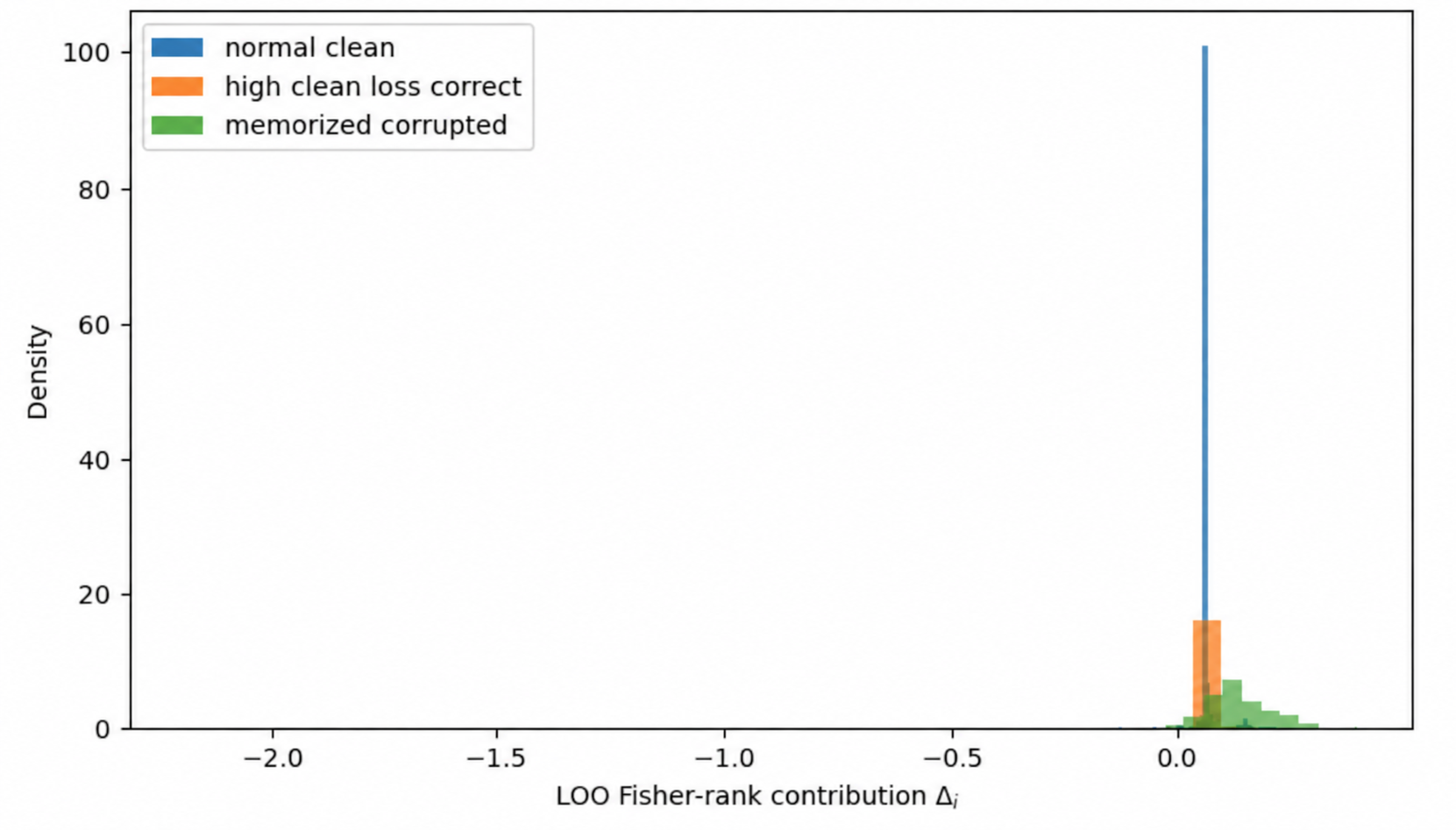}
    \caption{Seed 44}
\end{subfigure}

\caption{
Distribution of leave-one-out Fisher-rank contributions for the
clean-difficulty control across three random seeds. In all seeds,
memorized corrupted examples occupy the positive tail, whereas both clean
groups remain concentrated near zero.
}
\label{fig:control_histogram}
\end{figure}

\subsection{Eigenvalue Cutoff Sensitivity}

The effective rank computations in the main experiments discard eigenvalues not
exceeding an absolute tolerance of $10^{-12}$. Since an absolute tolerance is
not strictly invariant to rescaling of the gradient matrix, we evaluated the
sensitivity of the reported peak checkpoint effective ranks to alternative
data dependent thresholds.

For each of the five seeds in the CIFAR-10 SmallCNN, CIFAR-10 ResNet18,
CIFAR-10 ViT, and CIFAR-100 ResNet18 settings, we reconstructed the centered
last layer gradient matrix at the saved peak rank checkpoint. We compared the
original absolute threshold with relative thresholds
\[
\lambda_j > \tau \lambda_{\max},
\qquad
\tau \in \{10^{-14},10^{-12},10^{-10},10^{-8}\},
\]
as well as a machine precision based singular value tolerance.

At the native gradient scale, the maximum relative change in effective rank was
on the order of $10^{-6}$ across all evaluated settings. The retained scatter
trace fraction was essentially one in every case. We also rescaled each
gradient matrix by factors ranging from $10^{-4}$ to $10^{4}$ and observed only
negligible numerical variation. These results indicate that the fixed cutoff
does not materially affect the effective rank values reported at the peak rank
checkpoints.

\end{document}